\documentclass[multiauthors, onecolumn]{LaTeXclass/cohere}

\usepackage{amsmath,amsfonts,bm}
\usepackage{amsthm}
\newtheoremstyle{coheretheorem}
  {}{}{\itshape}{}{\bfseries}{.}{ }
  {\textcolor{DarkBlue}{\thmname{#1}\thmnumber{ #2}}\thmnote{ (#3)}}
\theoremstyle{coheretheorem}
\newtheorem{theorem}{Theorem}
\newtheorem{proposition}{Proposition}
\newtheorem{corollary}{Corollary}

\def\eqref#1{equation~\ref{#1}}

\def\1{\bm{1}}

\DeclareMathAlphabet{\mathsfit}{\encodingdefault}{\sfdefault}{m}{sl}
\SetMathAlphabet{\mathsfit}{bold}{\encodingdefault}{\sfdefault}{bx}{n}

\newcommand{\E}{\mathbb{E}}

\DeclareMathOperator*{\argmax}{arg\,max}

\usepackage{hyperref}
\usepackage{url}
\usepackage[dvipsnames]{xcolor}
\usepackage{algorithm}
\usepackage{algpseudocode}
\usepackage{wrapfig}
\usepackage{placeins}

\newcommand{\softrlvr}{\textsc{Soft-RLVR}}
\newcommand{\softsverl}{\textsc{Soft-SVeRL}}
\newcommand{\ifeval}{\textsc{IFEval}}

\newcommand{\genhl}[1]{%
  \begingroup
  \setlength{\fboxsep}{1pt}%
  \colorbox{NavyBlue!12}{\strut #1}%
  \endgroup
}
\newcommand{\verhl}[1]{%
  \begingroup
  \setlength{\fboxsep}{1pt}%
  \colorbox{ForestGreen!12}{\strut #1}%
  \endgroup
}
\newcommand{\parthl}[1]{%
  \begingroup
  \setlength{\fboxsep}{1pt}%
  \colorbox{BrickRed!12}{\strut #1}%
  \endgroup
}

\title{Soft-SVeRL: Self-Verified Reinforcement Learning with Soft Rewards}
\author{name={Saurabh Dash},affiliation={1}}
\author{name={Pierre Clavier},affiliation={2}}
\author{name={John Dang\lone},affiliation={1}}
\author{name={Matthias Galle\ltwo},affiliation={1}}
\author{name={Marzieh Fadaee},affiliation={1}}
\author{name={Ahmet Üstün},affiliation={2}}
\author{name={Beyza Ermis},affiliation={1}}

\affiliations{
    \item[1] Cohere Labs
    \item[2] Cohere
}
\corresponding[*]{\{saurabh,beyza\}@cohere.com}

\abstract{Reinforcement Learning from Verifiable Rewards (RLVR) has improved language models in domains such as mathematics and code, where correctness can be checked automatically. However, many important tasks are only partially verifiable: prompts contain multiple requirements, responses may satisfy some but not all of them, or no single reference answer might exist. 
We introduce \softrlvr, a framework for reinforcement learning from decomposed, learned verification signals. \softrlvr{} converts each prompt into a checklist of atomic requirements, scores candidate responses item by item with an LLM verifier, and trains on the resulting soft reward. 
Checklist-based rewards turn sparse pass/fail supervision into a denser partial-credit signal, but they also introduce a tradeoff: averaging item-level judgments can reduce verifier noise, while partial credit can reward incomplete responses. 
We formalize this tradeoff and identify conditions under which checklist-based verification gives a more reliable RL training signal than holistic verification. 
We further introduce \softsverl{}, a self-verifying variant of \softrlvr{} in which the policy also acts as the verifier. We show that self-verification is prone to reward inflation from overly permissive self-judgments, and that explicit stabilization is needed to prevent this collapse.
In a controlled instruction-following setting with rule-based ground-truth evaluation, checklist-based \softrlvr{} improves \ifeval{} by up to 11.1 points using only learned verifier rewards. Our experiments further show that verifier quality and checklist quality both affect downstream RL outcomes, and that explicit stabilization is essential for effective self-verification.
}

\begin{document}

\section{Introduction}

Reinforcement learning (RL) has become a powerful tool for improving language models when reliable reward signals are available. In domains such as mathematics~\citep{grpo,deepseekr1} and code generation~\citep{deepseekr1,wang2025coevolving}, where correctness can be verified accurately through symbolic checks or unit tests, Reinforcement Learning from Verifiable Rewards (RLVR) has led to dramatic gains in reasoning performance~\citep{deepseekr1,grpo}. However, these settings represent only a narrow subset of the real-world tasks for which we want to improve language models.

Many important applications do not allow exact verification. 
Complex instruction following is a central example: prompts often specify multiple requirements, such as formatting, keyword inclusion or exclusion, semantic manipulation, length limits, and answering a substantive question. A response may satisfy most requirements while failing one, and there is often no exact reference answer~\citep{zhou2023instructionfollowingevaluationlargelanguage}. 
Binary all-or-nothing rewards are therefore too sparse, while single holistic judgments compress many requirements into one noisy decision and do not reveal which requirements were satisfied or violated~\citep{bai2022constitutional,lee2023rlaif,gunjal2025rubrics}. This raises a central question: \textit{how can we train language models with RL when correctness is partial, noisy, and difficult to specify?}

We address this challenge with \softrlvr{}, a framework for reinforcement learning with learned checklist-based rewards. 
Our key observation is that many instruction-following tasks have an approximately decomposable reward structure: correctness arises from satisfying a set of identifiable requirements that can often be evaluated separately. 
\softrlvr{} makes this structure explicit by converting prompts into checklists of verifiable criteria, evaluating each checklist item with an LLM verifier, and computing reward as the fraction of satisfied items. 
This turns sparse pass/fail supervision into a graded signal that captures partial progress while keeping the reward tied to explicit criteria.

A key concern is that learned verifiers are imperfect: they may incorrectly approve unsatisfactory responses or reject valid ones. 
We show that checklist decomposition fundamentally reshapes the impact of this noise. 
While individual item-level judgments may be noisy, aggregating them can reduce reward variance, especially when verifier errors are not strongly correlated. At the same time, partial credit changes the training target by assigning non-zero reward to incomplete responses. 
We formalize this tradeoff and identify conditions under which checklist verification provides a more accurate estimate of the ideal RL update than holistic judgments.

We further introduce \softsverl{}, a self-verifying variant of \softrlvr{} where the policy model also serves as the verifier.
This removes the need for a separate online reward model and can reduce the inference footprint of RL training. 
However, it introduces a collapse mode in which the shared model increases measured reward by making the verifier more permissive rather than by generating better responses. We call this \emph{always-yes collapse}. 
To mitigate this, we introduce verifier co-training with gold and replayed examples, along with an anti-collapse penalty that discourages uniform inflation of item-level acceptance and stabilizes training.

We evaluate our approach in a controlled instruction-following setting based on \ifeval{}-style verifiable constraints~\citep{zhou2023instructionfollowingevaluationlargelanguage}. 
This setting provides rule-based ground-truth verification for evaluation and analysis, while training uses learned checklist-based rewards rather than the rule-based verifier. Our experiments show that checklist-based soft verification substantially improves instruction following, that verifier and checklist quality both affect downstream RL outcomes, and that self-verification requires explicit stabilization.

To summarize, our work makes the following key contributions:
\begin{itemize}
\item We introduce \softrlvr, a framework for reinforcement learning with decomposed LLM-based verifiers for partially verifiable instruction-following tasks, turning all-or-nothing pass/fail supervision into graded item-level feedback.
\item We characterize when averaging item-level verifier judgments yields a more reliable RL training signal than a single holistic judgment, and when partial-credit bias becomes harmful.
\item We introduce \softsverl{}, a self-verifying variant of \softrlvr{}, identify \emph{always-yes collapse}, and show that verifier co-training with gold and replay examples in addition to an anti-collapse penalty reduces reward inflation.
\item We demonstrate empirically in an \ifeval{}-style setting with rule-based ground-truth evaluation that learned checklist rewards substantially improve instruction following and that both verifier quality and checklist quality critically affect downstream RL performance.
\end{itemize}

\begin{figure}[!t]
    \centering
    \includegraphics[width=0.90\textwidth]{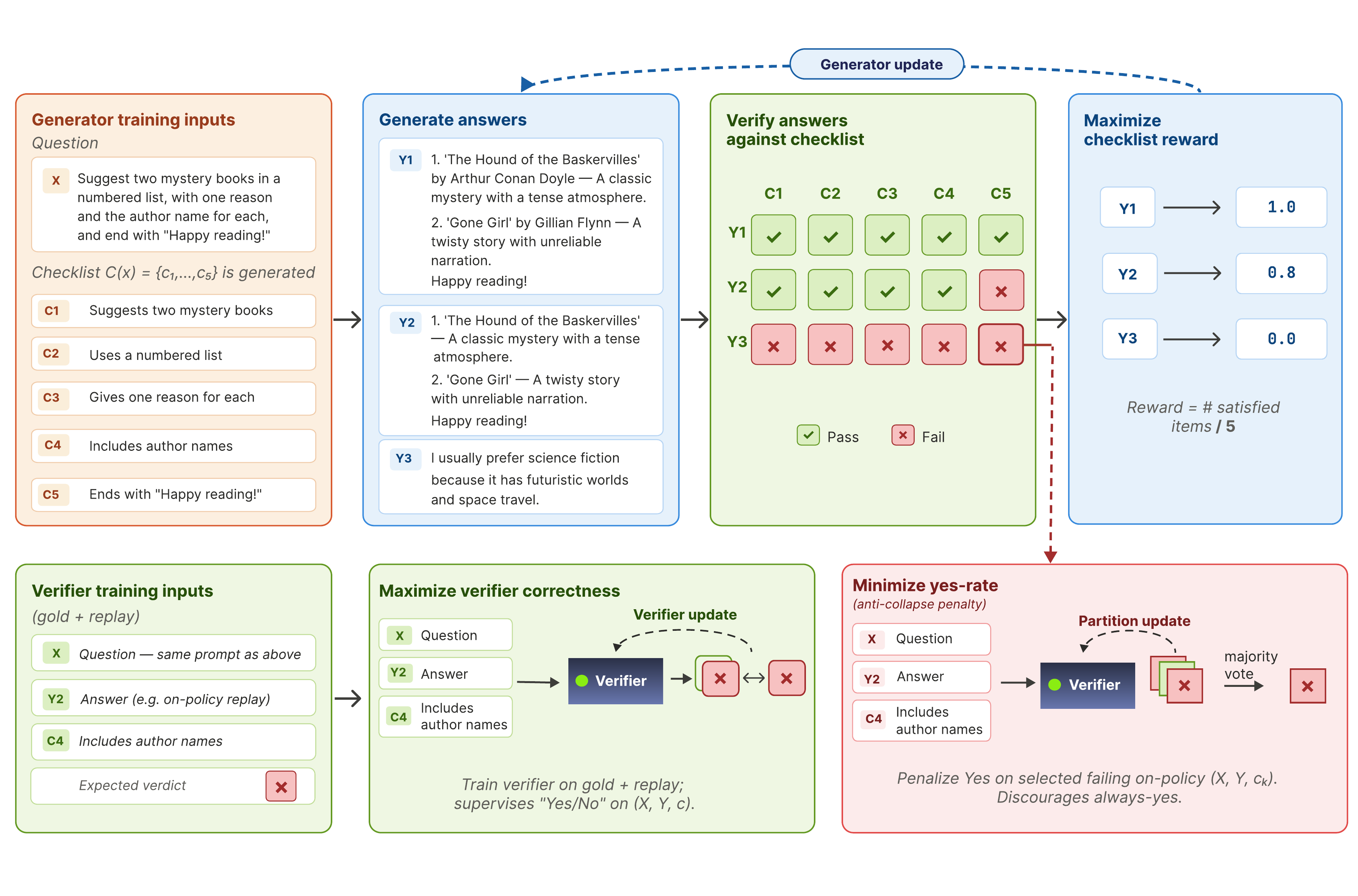}
    \caption{\textbf{Overview of \softsverl{}.}
    A prompt is decomposed into checklist criteria, candidate answers are scored item by item, and the resulting soft rewards update the generator. Because the generator also serves as the verifier, \softsverl{} adds verifier-side stabilization: labeled examples improve item-level judgments, while failing checklist contexts discourage inflated Yes predictions.
    }
    \label{fig:soft-sverl-overview}
\end{figure}

\section{Preliminaries}
\label{sec:background}

Let $\pi_\theta(y \mid x)$ denote a language model policy parameterized by $\theta$, where $x$ is a prompt and $y$ is a sampled completion. Given a reward function $R$ that scores completions, the goal of policy optimization is to maximize expected reward under the policy. In reference-based settings, this reward may depend on a reference answer $y^*$, giving the objective
\begin{equation}
    \label{eq:rl-objective}
    \argmax_{\theta}\; \E_{(x,y^*)\sim\mathcal{D},\; y\sim\pi_\theta(\cdot\mid x)}\bigl[R(y, y^*)\bigr],
\end{equation}
where $\mathcal{D} = \{(x_i, y_i^*)\}$ is a dataset of prompt-reference pairs.
In the settings we study, however, exact reference answers may be unavailable or insufficient, so the central question is how to construct a useful reward directly from the prompt and candidate response.

\subsection{Reinforcement learning with verifiable rewards (RLVR)}
When outputs admit exact verification, as in mathematics or code, a natural reward is binary correctness:
\begin{equation}
    \label{eq:binary-reward}
    R_{\mathrm{correct}}(y, y^*) = \mathbf{1}[y \equiv y^*],
\end{equation}
where equivalence may allow for formatting or other task-specific normalizations. More generally, the reward can be computed by a deterministic checker, such as a symbolic verifier or a unit-test suite. Because these rewards are reliable and not directly manipulable by the policy, they have enabled strong RLVR results in reasoning and code-generation settings~\citep{deepseekr1,grpo,kimik15,bercovich2025llamanemotron}.

Many instruction-following tasks, however, do not admit a single exact equivalence check. A response may satisfy some requirements but not others, making binary correctness too coarse. This motivates replacing exact verification with learned verification.

\subsection{Holistic learned verifiers}
For tasks without exact checkers, a learned verifier can provide a reward. Given a prompt $x$ and completion $y$, a verifier $\rho_\phi(x,y) \in \{0,1\}$ returns a binary judgment of whether the response satisfies the instruction:
\begin{equation}
    \label{eq:llm-verifier-reward}
    R_{\mathrm{ver}}(x,y;\phi) = \rho_\phi(x, y).
\end{equation}
In practice, $\rho_\phi$ is often implemented by prompting an LLM to produce a Yes/No decision.

A holistic learned verifier extends RL beyond exactly verifiable tasks, but it compresses evaluation into a single decision. For complex prompts, this decision can be noisy or underspecified: a response may satisfy some requirements and fail others, yet still receive only one pass/fail label. We therefore consider a more structured verifier that evaluates individual requirements separately.

\subsection{Checklist-based verification with learned verifiers}
Hard-to-verify instructions are often better described by a set of atomic criteria
$C(x)=\{c_1,\ldots,c_K\}$ derived from the prompt $x$. For each response $y$ and criterion $c_k$, a learned verifier $\rho_\phi$ judges whether $y$ satisfies $c_k$.

In our experiments, \(C(x)\) is generated offline by an LLM checklist generator conditioned on the instruction \(x\). The generator outputs an ordered numbered list of essential yes/no criteria, which we parse into checklist items and keep fixed throughout RL training. The exact prompt template is given in Appendix~\ref{app:checklist-prompt}; representative examples and failure modes are provided in Appendix~\ref{app:checklist-examples}.

In the simplest case, the verifier returns a binary item-level judgment
\begin{equation}
\rho_\phi(x,y,c_k) \in \{0,1\},
\end{equation}
and the checklist reward averages these judgments:
\begin{equation}
R_{\mathrm{soft}}(x,y;\phi)
=
\frac{1}{K}\sum_{k=1}^K \rho_\phi(x,y,c_k).
\end{equation}
This partial-credit formulation preserves learning signal for responses that satisfy some but not all criteria. 
For example, if a prompt asks for exactly three bullet points, a mention of Paris, and no use of the word \emph{travel}, then a response satisfying two of the three requirements receives reward $2/3$ rather than zero. At the same time, partial credit relaxes the strict-pass condition: incomplete responses can receive nonzero reward. Section~\ref{sec:checklist-mse} analyzes this tradeoff formally.

In practice, $\rho_\phi$ is implemented by a generative LLM. For each item $c_k$, we sample $J$ verifier outputs
\begin{equation}
z_{k1},\ldots,z_{kJ} \sim \pi_\phi(\cdot \mid x,y,c_k),
\end{equation}
and parse each output into a binary decision $d_{kj} = \mathrm{parse}(z_{kj}) \in \{0,1\}$.

We define the empirical pass rate
\begin{equation}
\hat{p}_k(x,y) = \frac{1}{J}\sum_{j=1}^J d_{kj},
\end{equation}
and obtain the aggregated verifier judgment
\begin{equation}
\rho_\phi(x,y,c_k) = \mathbf{1}[\hat{p}_k(x,y)\geq \tau].
\end{equation}
Other aggregators, such as majority vote, shaped-averaging, or median aggregation, are possible; throughout the paper, $\rho_\phi(x,y,c_k)$ denotes the parsed and aggregated checklist-item judgment.

\section{Reinforcement Learning with Checklist-Based Soft Rewards}
\label{sec:soft-rl}

Given a prompt $x$, a checklist $C(x)=\{c_1,\ldots,c_K\}$, and a learned verifier $\rho_\phi$, \softrlvr{} trains the policy $\pi_\theta(y\mid x)$ to maximize the expected checklist reward:
\begin{equation}
\arg\max_\theta
\mathbb{E}_{x\sim D_X,\, y\sim \pi_\theta(\cdot\mid x)}
\left[
r(x,y,C(x),\rho_\phi)
\right].
\end{equation}
where $r$ is computed from item-level verifier judgments. In the simplest case, $r=R_{\mathrm{soft}}$, the fraction of checklist items judged as satisfied. This gives the policy partial credit for satisfying some requirements even when the full response is not a strict success. Responses that would both receive zero reward under a strict pass/fail verifier can therefore receive different rewards depending on how much of the instruction they satisfy.

This formulation is analogous to RLVR for code, where a generated program is evaluated against multiple unit tests. The key difference is that checklist items are evaluated by a learned LLM verifier rather than executable tests. This makes the reward applicable beyond exactly verifiable domains, but it also introduces two sources of error. First, each item-level verifier judgment may be noisy. Second, partial credit can reward incomplete responses. The rest of this section analyzes when the benefit of decomposition outweighs this additional bias.

\subsection{Bias--variance tradeoff of checklist verification}
\label{sec:checklist-mse}

Checklist verification changes the error profile of the reward. A holistic verifier produces one noisy judgment for the whole response, while a checklist verifier produces $K$ item-level judgments and averages them. Averaging can reduce reward variance, especially when item-level errors are weakly correlated. However, the soft reward also relaxes the strict success target: a response that fails one requirement may still receive positive reward. Thus, checklist verification can be guaranteed to improve the training signal only when the variance reduction from decomposition is large enough to offset the bias introduced by partial credit. 
We formalize this tradeoff below.

\subsubsection{Setup}

Let $\mathbf{Y}^*(x,y) = [Y_1^*,\dots,Y_K^*]^\top \in \{0,1\}^K$ denote the latent checklist vector, where $Y_k^*=1$ means that response $y$ satisfies checklist item $c_k$.
We define the strict success target as
\begin{equation}
S^*(x,y) := \prod_{k=1}^K Y_k^* ,
\label{eq:strict-success-target}
\end{equation}
and the partial-credit target as
\begin{equation}
\bar S^*(x,y) := \frac{1}{K}\sum_{k=1}^K Y_k^* .
\label{eq:partial-credit-target}
\end{equation}
The relaxation gap
\begin{equation}
\Delta^*(x,y) := \bar S^*(x,y)-S^*(x,y) \ge 0
\label{eq:relaxation-gap}
\end{equation}
measures how much partial credit rewards incomplete responses beyond the strict success target.
We take strict success as the ideal target because our controlled evaluation uses rule-based pass/fail instruction satisfaction; partial credit is treated as a training relaxation that may improve optimization but can introduce bias relative to this target.

Let $J$ denote the output of a single holistic verifier and $J_k$ the output of the checklist verifier for item $k$. We define the averaged checklist judgment as
\begin{equation}
\bar J := \frac{1}{K}\sum_{k=1}^K J_k .
\label{eq:avg-checklist-judgment}
\end{equation}
We also define the score function as
\begin{equation}
s_\theta(x,y) := \nabla_\theta \log \pi_\theta(y \mid x),
\label{eq:score-function}
\end{equation}
and let $b(x)$ be a baseline.

We compare the update induced by the learned verifier to the update that would be obtained from the true strict-success signal. The ideal, single-verifier, and checklist estimators are
\begin{equation}
\begin{aligned}
g^*(x,y) &:= \bigl(S^*(x,y)-b(x)\bigr)s_\theta(x,y), \\
g_{\mathrm{single}}(x,y) &:= (J-b(x))\,s_\theta(x,y), \\
g_{\mathrm{chk}}(x,y) &:= (\bar J-b(x))\,s_\theta(x,y).
\end{aligned}
\label{eq:gradient-estimators}
\end{equation}

Finally, let $p$ and $q$ denote the true positive and true negative rates of the holistic verifier with respect to the strict target $S^*$.
Let $p'$ and $q'$ denote the corresponding item-level true positive and true negative rates of the checklist verifier with respect to each latent item label $Y_k^*$.
We write
\begin{equation}
    \alpha := p+q-1,
    \qquad
    \alpha' := p'+q'-1 .
    \label{eq:alpha-definitions}
\end{equation}
Equivalently, the holistic verifier satisfies
\begin{equation}
    \E[J \mid S^*] = (1-q)+\alpha S^*
    \label{eq:holistic-verifier-expectation}
\end{equation}
and each checklist verifier satisfies
\begin{equation}
    \E[J_k \mid Y_k^*] = (1-q')+\alpha'Y_k^* .
    \label{eq:checklist-verifier-expectation}
\end{equation}
These quantities let us compare two sources of error: variance from noisy verifier judgments and bias from replacing strict success $S^*$ with partial credit $\bar S^*$.

\begin{proposition}[Checklist variance under conditional independence]
\label{prop:checklist-variance}
Under the setup above, if $J_1,\dots,J_K$ are conditionally independent given $(x,y)$, then
\begin{equation}
\mathrm{Var}(\bar J \mid x,y)
=
\frac{1}{K^2}\sum_{k=1}^K \mu_k(x,y)\bigl(1-\mu_k(x,y)\bigr)
\le \frac{1}{4K},
\end{equation}
where $\mu_k(x,y) := (1-q')+\alpha' Y_k^*(x,y)$, and therefore
\begin{equation}
\mathbb{E}\!\left[
\|g_{\mathrm{chk}}-\mathbb{E}[g_{\mathrm{chk}}\mid x,y]\|_2^2
\,\middle|\, x,y
\right]
\le
\frac{1}{4K}\|s_\theta(x,y)\|_2^2.
\end{equation}
\end{proposition}

The proof is in Appendix~\ref{app:proof-prop-checklist-variance}.
The single-verifier estimator has conditional variance
$\mu_{\mathrm{single}}(1-\mu_{\mathrm{single}})\|s_\theta\|_2^2 = O(1)$,
where $\mu_{\mathrm{single}}(x,y):=(1-q)+\alpha S^*(x,y)$.
Thus, under conditional independence, averaging $K$ item-level verifier judgments reduces the variance term by a factor of $1/K$.

The conditional independence assumption is idealized: verifier errors across checklist items may be correlated in practice. The proposition should therefore be read as identifying the variance-reduction mechanism of decomposition; positive error correlation reduces the gain, while weakly correlated item errors make the gain closer to the $1/K$ rate.

Variance reduction alone is not sufficient. Checklist rewards can also be biased because partial credit assigns nonzero reward to incomplete responses. 
The bias is harmless on correct outputs if item-level sensitivity is at least as high as holistic sensitivity. On incorrect outputs, however, partial credit creates an additional error proportional to the relaxation gap $\Delta^*$. This error must be small relative to any specificity gain from item-level verification. Appendix~\ref{app:proof-prop-checklist-bias} formalizes this bias comparison.

Combining the variance and bias terms yields the following sufficient conditions for checklist verification to improve the mean-squared error (MSE) of the policy-gradient estimator. Here MSE measures how close the verifier-induced update is to the ideal update $g^*$: lower MSE means that the learned reward produces an update that is better aligned with the update from the true strict-success signal.

\begin{theorem}[Sufficient conditions for gradient MSE improvement]
\label{thm:checklist-mse}
Suppose the single verifier satisfies $\mathbb{E}[J \mid S^*] = (1-q)+\alpha S^*$ with
$\alpha := p+q-1$, each checklist verifier satisfies
$\mathbb{E}[J_k \mid Y_k^*] = (1-q')+\alpha'Y_k^*$ with $\alpha' := p'+q'-1$, and
$J_1,\dots,J_K$ are conditionally independent given $(x,y)$.
Then
\begin{equation}
\mathbb{E}\!\left[\|g_{\mathrm{chk}}-g^*\|_2^2 \,\middle|\, x,y\right]
\le
\mathbb{E}\!\left[\|g_{\mathrm{single}}-g^*\|_2^2 \,\middle|\, x,y\right]
\end{equation}
whenever either
\begin{equation}
S^*(x,y)=1 \qquad\text{and}\qquad p' \ge p,
\end{equation}
or
\begin{equation}
S^*(x,y)=0 \qquad\text{and}\qquad
\Bigl(1-q'+\alpha'\Delta^*(x,y)\Bigr)^2 + \frac{1}{4K} \le 1-q.
\end{equation}
\end{theorem}

The proof is in Appendix~\ref{app:proof-thm-checklist-mse}.
On correct outputs ($S^*=1$), checklist verification gives an update closer to the ideal update whenever item-level sensitivity is at least as high as holistic sensitivity ($p'\ge p$), with an additional benefit from the $1/K$ variance reduction that comes from averaging $K$ conditionally independent verifier calls.
On incorrect outputs ($S^*=0$), partial-credit relaxation introduces extra bias proportional to $\alpha'\Delta^*$.
The condition in the theorem requires this bias to be small enough for variance reduction to offset it.

\textbf{Takeaway.}
Checklist verification is most useful when decomposition makes the reward less noisy without substantially changing what the reward is trying to optimize. Larger checklists can reduce variance, but they cannot by themselves fix a poor decomposition or a verifier that systematically approves incomplete responses. Thus, checklist quality and verifier specificity are central: the checklist items must capture meaningful requirements, and the verifier must reject failures often enough that partial credit does not become reward inflation.

This analysis leads to two empirical predictions that we test in Section~\ref{sec:checklist-vs-holistic}. First, checklist verification should help most when holistic verifier judgments are noisy. Second, checklist quality should matter: poorly decomposed criteria can increase the relaxation gap or make item-level judgments less reliable, reducing the benefit of decomposition. Appendix~\ref{app:additional-checklist-mse} gives additional corollaries that make the checklist-size thresholds explicit and show that larger $K$ alone cannot guarantee improvement when partial-credit bias is too large.

\section{\softsverl: Self-Verified Soft RLVR}
\label{sec:self-verification}

The checklist-based formulation in Section~\ref{sec:soft-rl} assumes a learned verifier $\rho_\phi$ that is fixed or trained separately from the policy. We now consider the shared-parameter case. 
In the self-verification setting, we set $\phi=\theta$, so the same model serves as both generator and verifier. This setting has two practical advantages: it removes the need for a separate online reward model, and it can be applied when no stronger external judge is available. It can also reduce training cost by avoiding calls to a separate verifier model.

\begin{wrapfigure}{r}{0.48\textwidth}
    \centering
    \vspace{-3mm}
    \includegraphics[width=\linewidth]{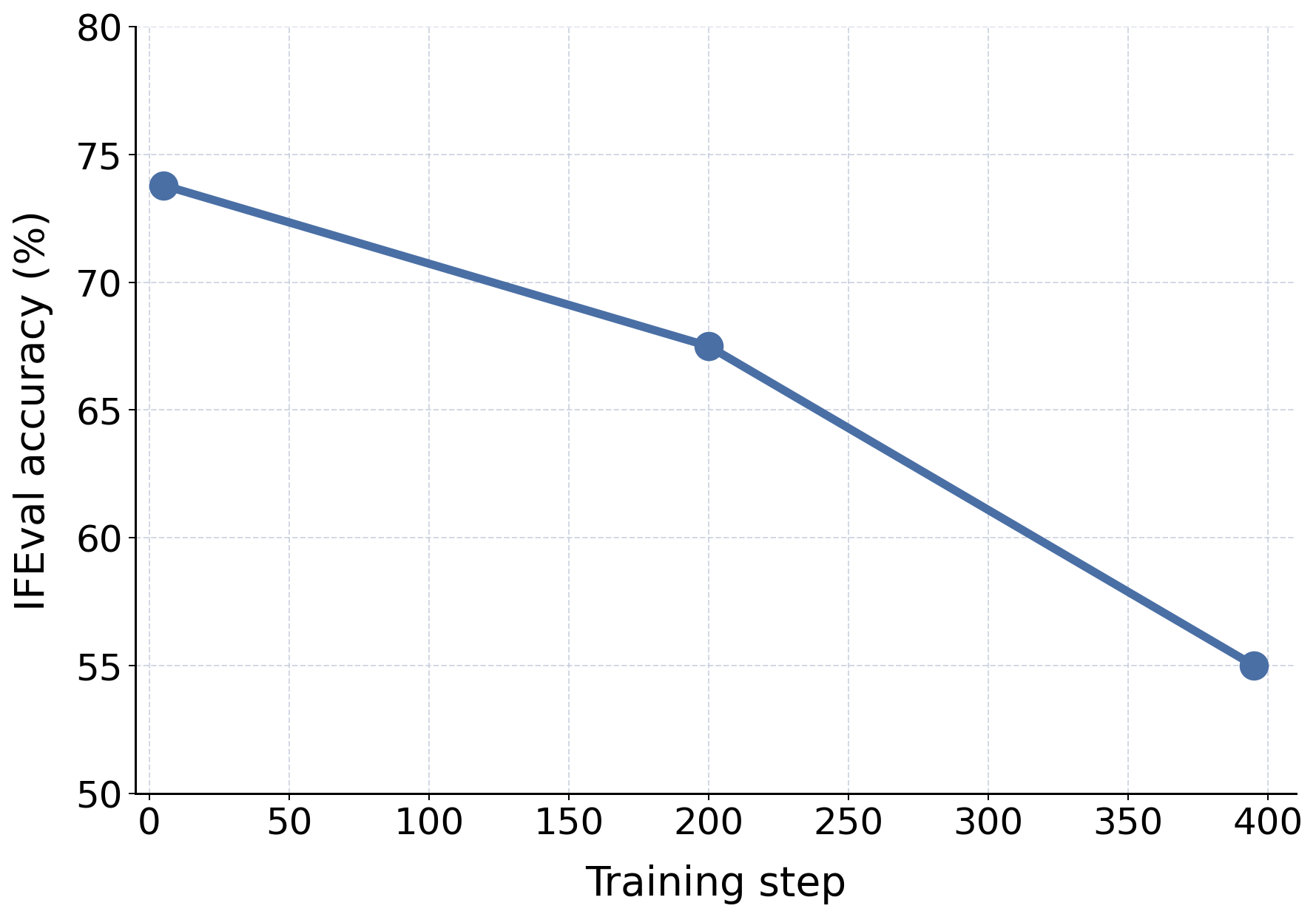}
    \caption{\textbf{Naive self-verification training.} The generator and verifier share parameters and are optimized jointly using the verifier's own reward signal. \ifeval{} drops from 73.9\% at initialization to 55.1\% by step~400, consistent with always-yes collapse: measured reward improves while task accuracy degrades.}
    \label{fig:reward-collapse}
    \vspace{-3mm}
\end{wrapfigure}

However, sharing parameters also changes the optimization problem. With a fixed external verifier, the reward function is outside the policy update: the policy can increase reward only by producing outputs that the verifier already scores highly. With self-verification, the reward function is implemented by the same parameters being updated. The model can therefore increase measured reward either by generating better responses or by making its own verifier more permissive.

This creates the central challenge for self-verification: the measured reward can improve even when task performance degrades. Section~\ref{sec:why-nontrivial} describes this failure mode, which we call \emph{always-yes collapse}, and Section~\ref{sec:stabilize} introduces the stabilizers used in \softsverl{}.

\subsection{Why self-verification is non-trivial}
\label{sec:why-nontrivial}

The resulting failure mode is \emph{always-yes collapse}, in which the verifier says Yes for almost every response and checklist item, whether or not the response satisfies the requirement. In terms of the empirical vote fraction, this corresponds to $\hat p_k(x,y)\approx 1$ for nearly every $(x,y,c_k)$. 
Once this happens, low-quality generations receive inflated rewards and the generator loses the gradient signal needed to improve. Two mechanisms make this collapse especially likely under self-verification.

\textbf{Correlated priors.} First, the verifier starts from the same inductive biases as the generator. The responses that are easiest for the generator to produce are also the responses on which the verifier may be most overconfident. Thus, the relevant question is not only whether the verifier is calibrated on a held-out set, but whether it is calibrated on the distribution of outputs induced by the current policy. This distinction matters because self-verification repeatedly scores model-generated responses, precisely where shared biases between generation and verification are most likely to appear.

\textbf{Gradient coupling.} Second, the shared objective no longer distinguishes between improving the generator and making the verifier more permissive. With a separate verifier, increasing expected reward requires the policy to put more probability on responses that the verifier already scores highly. When $\phi=\theta$, the same measured reward can also increase because the verifier assigns higher scores to responses the policy already produces. Both effects improve the training objective, but only the first corresponds to better task performance. Since both effects arise from updates to the same parameters, ordinary policy optimization has no reason to prefer task improvement over reward inflation. The shared model can therefore drift toward over-approval unless the objective explicitly counteracts this tendency.

Figure~\ref{fig:reward-collapse} shows the empirical signature of this pathology: naive shared-parameter training drops \ifeval{} by roughly 19 points over 400 training steps, well below the initial policy.
This confirms that self-verification without stabilization can degrade task performance rather than improve it.

\subsection{Stabilizing self-verification}
\label{sec:stabilize}

Stable self-verification requires addressing both failure modes above. Correlated priors make the verifier unreliable on the model-generated responses it must score, while gradient coupling allows the shared update to increase measured reward by making the verifier more permissive. \softsverl{} first reduces reward noise by aggregating multiple verifier samples for each response--checklist pair, then adds two stabilizing signals to the shared objective: verifier co-training, which anchors the verifier on labeled examples, and a partition-style Yes-rate penalty, which directly discourages reward inflation on failing items.

Figure~\ref{fig:soft-sverl-overview} gives a high-level overview of how generation, checklist verification, and verifier-side stabilization interact in \softsverl{}.

\textbf{Repeated verification for generator rewards.}
Because the generator reward is computed from sampled verifier judgments, a single verifier rollout can inject substantial reward noise. For each response--checklist pair, we sample $J$ verifier traces and compute the empirical Yes-rate $\hat p_k(x,y)$. This aggregation gives the generator a more stable item-level reward signal and also provides a confidence statistic for deciding which on-policy examples are reliable enough to reuse for verifier training.

\textbf{Verifier co-training with gold and replay data.}
The shared model is additionally trained as a verifier on a mixture of (i) \emph{gold} examples from a fixed verifier-improvement buffer, consisting of prompt--response--criterion triples $(x,y,c)$ with known item-level labels $\ell\in\{0,1\}$, and (ii) \emph{replay} examples sampled from the current generator policy $\pi_\theta$ and pseudo-labeled at the checklist-item level. 
This term primarily addresses correlated priors. Gold examples provide an independent anchor for what checklist satisfaction means, so the verifier does not define correctness only through what the current policy tends to generate. Replay examples move this supervision onto the distribution used for reward computation, including current-policy failures that a shared verifier may otherwise approve. 
For replay tagging, we use separate thresholds $\tau_+$ and $\tau_-$. Because self-verification is especially vulnerable to false positives, replay labels are added conservatively: items with $\hat p_k(x,y)\ge\tau_+$ are added as positives, items with $\hat p_k(x,y)\le\tau_-$ are added as negatives, and intermediate items are not used as hard verifier labels. The coefficient $\lambda_v$ controls how strongly these supervised verifier gradients constrain the shared update.

\textbf{Partition-style Yes-rate penalty.}
Verifier co-training alone does not fully close the reward-inflation channel, because the generator objective can still benefit from shared-parameter updates that raise verifier Yes-rates on the policy's own samples. \softsverl{} therefore adds a penalty on a selected subset of failed checklist contexts. We call this a partition-style penalty because its gradient is a selected, biased finite-sample approximation to the negative policy-sample term obtained by differentiating the log partition function of the KL-regularized optimal policy; Appendix~\ref{app:partition-loss-motivation} shows the derivation. The bias comes from replacing the full expectation over reward-induced policy samples and checklist items with current-policy samples from a selected failure subset; this selection is intentional, since applying the correction to all on-policy items could also suppress genuine successes. Unlike replay labels, this penalty does not require a confident negative threshold because it is not used as hard supervision. 
We instead select mixed-vote items from responses that are not strict successes: $\hat p_k(x,y)>0$ and $s(x,y)<1$. This excludes strict successes and unanimous-No items, while retaining cases where the verifier produces one or more Yes votes on a response that failed consensus. These Yes traces reveal permissive verifier decisions that self-reward can amplify. The partition penalty lowers the Yes-rate on these selected failures, directly opposing updates that would make the verifier more permissive. The coefficient $\lambda_p$ controls the strength of this anti-collapse pressure. Because the penalty reuses the verifier traces collected during reward scoring, it does not require additional verifier forward passes.

Combining the generator reward with these two stabilizing signals gives the shared objective
\begin{equation}
\mathcal{J}(\theta)
=
\mathcal{J}_{\mathrm{gen}}(\theta)
+ \lambda_v \mathcal{J}_{\mathrm{ver}}(\theta)
- \lambda_p \mathcal{J}_{\mathrm{part}}(\theta).
\label{eq:shared_objective}
\end{equation}
Here $\mathcal{J}_{\mathrm{gen}}$ trains the generator under the soft checklist reward,
$\mathcal{J}_{\mathrm{ver}}$ trains the verifier on gold and replay labels, and
$\mathcal{J}_{\mathrm{part}}$ penalizes the Yes-rate on the failing partition and is subtracted as an anti-collapse penalty.
Algorithm~\ref{alg:soft-sverl} summarizes the resulting shared-parameter training loop.
\begin{algorithm}[!t]
\small
\caption{Soft-SVeRL}
\label{alg:soft-sverl}
\begin{algorithmic}[1]
\setlength{\itemsep}{0pt}
\setlength{\parsep}{0pt}
\Statex \textbf{Inputs:} Input distribution $\mathcal{D}_X$; gold verifier buffer
$\mathcal{D}_{\mathrm{Gold}}$; replay verifier buffer
$\mathcal{D}_{\mathrm{Replay}}$; checklist generator $\mathcal{C}(\cdot)$;
candidate count $G$; verifier vote count $J$; thresholds $\tau_+,\tau_-$;
partial-credit scale $\beta$; weights
$\lambda_v,\lambda_p$
\Statex \textbf{Model:} Shared model $\theta$ used for generation $\pi_\theta$ and verification $\rho_\theta$
\For{$t = 1,\dots,T$}
    \State \genhl{Sample $x \sim \mathcal{D}_X$ and checklist
    $\mathcal{C}(x)=\{c_k\}_{k=1}^K$}
    \State \genhl{Sample candidates $y_1,\dots,y_G \sim \pi_\theta(\cdot \mid x)$}
    \ForAll{\genhl{$(i,k,j)\in [G]\times[K]\times[J]$}}
        \State \genhl{Sample $z_{ikj} \sim \rho_\theta(\cdot \mid x,y_i,c_k)$ and set
        $d_{ikj} \gets \operatorname{parse}(z_{ikj})$}
    \EndFor
    \State \genhl{Compute $\hat p_{ik} \gets \frac{1}{J}\sum_{j=1}^J d_{ikj}$ for all $(i,k)$}
    \State \genhl{Set $\hat \ell_{ik}\gets \mathbf{1}[\hat p_{ik}\ge\tau_+]$ for all $(i,k)$}
    \Comment{\genhl{pass label}}
    \For{\genhl{$i=1,\dots,G$}}
        \State \genhl{Set $s_i \gets \frac{1}{K}\sum_{k=1}^K \hat \ell_{ik}$}
        \Comment{\genhl{soft checklist score}}
        \State \genhl{Set $R_i \gets \mathbf{1}[s_i=1] + \beta s_i\,\mathbf{1}[s_i<1]$}
        \Comment{\genhl{generator reward}}
    \EndFor
    \ForAll{\verhl{$(i,k)\in [G]\times[K]$}}
        \If{\verhl{$\hat p_{ik}\ge\tau_+$}}
            \State \verhl{Insert $(x,y_i,c_k,\ell=1)$ into
            $\mathcal{D}_{\mathrm{Replay}}$}
        \ElsIf{\verhl{$\hat p_{ik}\le\tau_-$}}
            \State \verhl{Insert $(x,y_i,c_k,\ell=0)$ into
            $\mathcal{D}_{\mathrm{Replay}}$}
        \EndIf
    \EndFor
    \State \parthl{Select subset $\mathcal{S}(x)\subseteq [G]\times[K]$ where
    $\mathcal{S}(x)=\{(i,k): \hat p_{ik}>0 \text{ and } s_i<1\}$}
    \State \verhl{Sample verifier tuples $(x,y,c,\ell)
    \sim \mathcal{D}_{\mathrm{Gold}} \cup \mathcal{D}_{\mathrm{Replay}}$}
    \For{\verhl{$g=1,\dots,G$}}
        \State \verhl{Sample $z_g \sim \rho_\theta(\cdot \mid x,y,c)$}
        \State \verhl{Set $R_g^{\mathrm{ver}} \gets \mathbf{1}[\operatorname{parse}(z_g)=\ell]$}
        \Comment{\verhl{checklist verifier reward}}
    \EndFor
    \ForAll{\parthl{$(i,k)\in\mathcal{S}(x)$}}
        \For{\parthl{$j=1,\dots,J$}}
            \State \parthl{Set $R_{ikj}^{\mathrm{part}}\gets d_{ikj}$}
            \Comment{\parthl{Yes-rate penalty}}
        \EndFor
    \EndFor
    \State Update $\theta$ to maximize
    $\mathcal{J}(\theta)=\mathcal{J}_{\mathrm{gen}}(\theta)
    + \lambda_v \mathcal{J}_{\mathrm{ver}}(\theta)
    - \lambda_p \mathcal{J}_{\mathrm{part}}(\theta)$
    \Comment{shared update}
\EndFor
\end{algorithmic}
\end{algorithm}

\section{Experiment Setup}
\label{sec:experiments}

\textbf{Training data.}
We train on prompts from the RL split of the Llama-Nemotron-Post-Training Dataset~\citep{bercovich2025llamanemotron}, which contains 56{,}339 instruction-following prompts. Each prompt includes one or more explicit constraints drawn from the 25 verifiable instruction types defined by \ifeval{}~\citep{zhou2023instructionfollowingevaluationlargelanguage}, such as formatting requirements, keyword inclusion or exclusion, length limits, and structural specifications. Each constraint in this dataset can be checked by a deterministic rule-based verifier, which we treat as ground truth for evaluation. Because the constraint universe matches \ifeval{}, training with this oracle verifier uses the same family of rule-based checks as evaluation, eliminating verifier train--test mismatch.

This controlled training and evaluation setup is useful for our analysis because it separates reward-learning effects from evaluation ambiguity. The deterministic verifier lets us evaluate learned checklist rewards against exact rule-based labels, measure verifier agreement with ground-truth constraints, and test the empirical predictions of the MSE analysis in Section~\ref{sec:checklist-mse}.

\textbf{Training details.}
All experiments start from Command R7B%
\footnote{\href{https://huggingface.co/CohereLabs/c4ai-command-r7b-12-2024}{https://huggingface.co/CohereLabs/c4ai-command-r7b-12-2024}},
further adapted with the math-and-code SFT data used in the Command A Reasoning recipe~\citep{cohere2025commanda}%
\footnote{\href{https://cohere.com/blog/command-a-reasoning}{https://cohere.com/blog/command-a-reasoning}}.
We generate prompt checklists offline using DeepSeek-R1-05-28%
\footnote{\href{https://api-docs.deepseek.com/news/news250528}{https://api-docs.deepseek.com/news/news250528}};
the checklist-generation prompt is given in Appendix~\ref{app:checklist-prompt}.
For external-verifier experiments, we use GPT-OSS-20B and GPT-OSS-120B as the main LLM verifiers. These models provide a practical tradeoff between instruction-following quality and verification throughput, which is important because reward computation requires scoring multiple checklist items per sampled response. We also include additional verifier choices, such as the initial policy, in the verifier-quality ablations. The verifier prompt is shown in Appendix~\ref{app:verification-prompt}. For verifier models that support adaptive thinking, we use the maximum thinking-budget setting.

We optimize the policy with Group Relative Policy Optimization (GRPO)~\citep{grpo} using a FAX-based asynchronous RL training framework~\citep{yoo2022fax}. Following DAPO~\citep{dapo}, we use asymmetric clipping bounds. For external verifiers, we only use a single pass ($J=1$) to generate the decision. We sweep learning rates and report the best-performing run for each setting, with all other hyperparameters listed in Appendix~\ref{app:training-hparams}.

\textbf{Evaluation.}
We evaluate instruction following on \ifeval{}~\citep{zhou2023instructionfollowingevaluationlargelanguage}, which measures compliance with automatically checkable natural-language instructions. \ifeval{} covers 25 verifiable instruction types, including length, keyword, formatting, casing, punctuation, and structural constraints, and the released English dataset contains 541 prompts with one or more instructions per prompt. This makes it a useful testbed for precise instruction following without requiring subjective human evaluation.

As a sanity check for capability regression, we also evaluate on three math reasoning benchmarks outside the instruction-following training domain: MATH500~\citep{hendrycks2021measuring}, AIME 2024, and AIME 2025. Because the training data contains no math problems, these benchmarks let us check whether RL training on instruction-following data degrades, preserves, or incidentally improves mathematical reasoning.

\textbf{Baselines.}
We compare against the initial Command R7B policy before RL training.
For ablations where rule-based training is computationally feasible, we also include an oracle RLVR reference that uses the deterministic \ifeval-style verifier on the same training data.
This oracle removes both checklist-generation error and verifier noise, and therefore serves as an upper reference for learned checklist rewards rather than as the main deployed setting.

\section{Results}
\label{sec:results}

We evaluate three questions: whether learned checklist rewards can improve instruction following, how performance depends on verifier and checklist quality, and whether self-verification can be stabilized when the generator and verifier share parameters. We first report the main \softrlvr{} results using external LLM verifiers, then analyze verifier and checklist quality, followed by self-verification and a comparison between checklist-based and holistic verification.

Table~\ref{tab:main-results} summarizes the main external-verifier results. We train Command R7B with checklist-based RL using DeepSeek R1 checklists and LLM-based soft verification, then evaluate on IFEval, our target instruction-following benchmark, and on three out-of-domain math benchmarks: MATH500, AIME 2024, and AIME 2025. Unless mentioned otherwise, we use a single verifier pass ($J=1$) for grading the checklist items and do not train the verifier.

\begin{table}[!h]
\centering
\caption{\textbf{Results on instruction following and math reasoning.} The baseline is Command R7B before RL training. Soft-RLVR rows use DeepSeek R1 checklists and checklist-based soft rewards from a 20B or 120B LLM verifier. The rule-based verifier is an oracle reference. Bold indicates the best learned-verifier result. All models are evaluated at training step 500, the No RL row is the initial policy. Scores are percentages. Math benchmarks are held out from RL training, so any math gains measure transfer rather than direct optimization.}
\label{tab:main-results}
\scalebox{0.92}{
\begin{tabular}{lcccc}
\toprule
\textbf{Reward} & \textbf{\ifeval{}} & \textbf{MATH500} & \textbf{AIME24} & \textbf{AIME25} \\
\midrule
No RL & 73.89 & 87.60 & 35.67 & 27.33 \\
\midrule
Soft-RLVR, GPT-OSS-20B  & 84.20 & 89.60 & 39.33 & \textbf{28.67} \\
Soft-RLVR, GPT-OSS-120B & \textbf{85.00} & \textbf{90.20} & \textbf{40.00} & 28.50 \\
\midrule
Oracle rule-based & 88.92 & 87.80 & 36.83 & 28.00 \\
\bottomrule
\end{tabular}
}
\end{table}

Checklist-based soft verification substantially improves instruction following. \ifeval{} increases from 73.89 to 84.20 with the GPT-OSS-20B verifier and to 85.00 with the GPT-OSS-120B verifier, corresponding to gains of +10.3 and +11.1 points, respectively. 
These gains are obtained using generated checklists and imperfect LLM-based verifier rewards, without access during training to the rule-based verifier or to the true constraint annotations associated with each prompt. 
This supports one of our central claims: generated checklists and learned verifiers can drive large RL gains in instruction following, even without rule-based rewards during training.
(Appendix~\ref{app:checklist-distribution} reports the distribution of generated checklist lengths used in these runs.)

The oracle rule-based verifier provides an upper reference for this controlled setting: it obtains the highest \ifeval{} score, as expected, but relies on deterministic rule-based ground-truth checklists that might not be available in practice.

Despite training exclusively on instruction-following prompts and no math data, \softrlvr{} also improves math benchmark performance, indicating domain transfer: MATH500 improves by 2.0--2.6 points and AIME 2024 improves by 3.7--4.3 points across verifier scales. AIME 2025 gains are modest, around one point, and should be interpreted cautiously given the small benchmark size of 30 problems. The absence of math regression is notable in itself, as targeted post-training can degrade out-of-domain capabilities; we provide an SFT comparison in Appendix~\ref{app:sft-vs-rl}.

Increasing the verifier scale from GPT-OSS-20B to GPT-OSS-120B gives consistent but modest additional gains: +0.8 on \ifeval{}, +0.6 on MATH500, and +0.7 on AIME 2024, while AIME 2025 remains essentially unchanged. This suggests that the checklist-based reward formulation is robust to moderate variation in verifier quality. We analyze the relationship between verifier quality and downstream training outcomes more systematically in Section~\ref{sec:verifier-quality}.

\subsection{Disentangling Verifier and Checklist Quality}
\label{sec:verifier-quality}

\softrlvr{} depends on two learned components: the checklist used to decompose the prompt into criteria, and the verifier used to judge whether each criterion is satisfied. 
To disentangle their effects, we keep the RL pipeline fixed while varying both the checklist source and the verifier model. 
We compare checklists generated by the Initial Policy and DeepSeek R1, as well as a metadata setting that directly uses the ground-truth constraints stored in the training data as checklist items. 
The metadata setting removes checklist-generation error allowing us to isolate the impact of verifier imperfections. 
As an upper-bound reference, we also include training with the deterministic \ifeval{} rule-based verifier, which removes both checklist error and verifier noise. To reduce computational costs, all runs in this ablation are evaluated at training step 200 rather than step 500.

\begin{figure}[!h]
    \centering
    \includegraphics[width=0.8\textwidth]{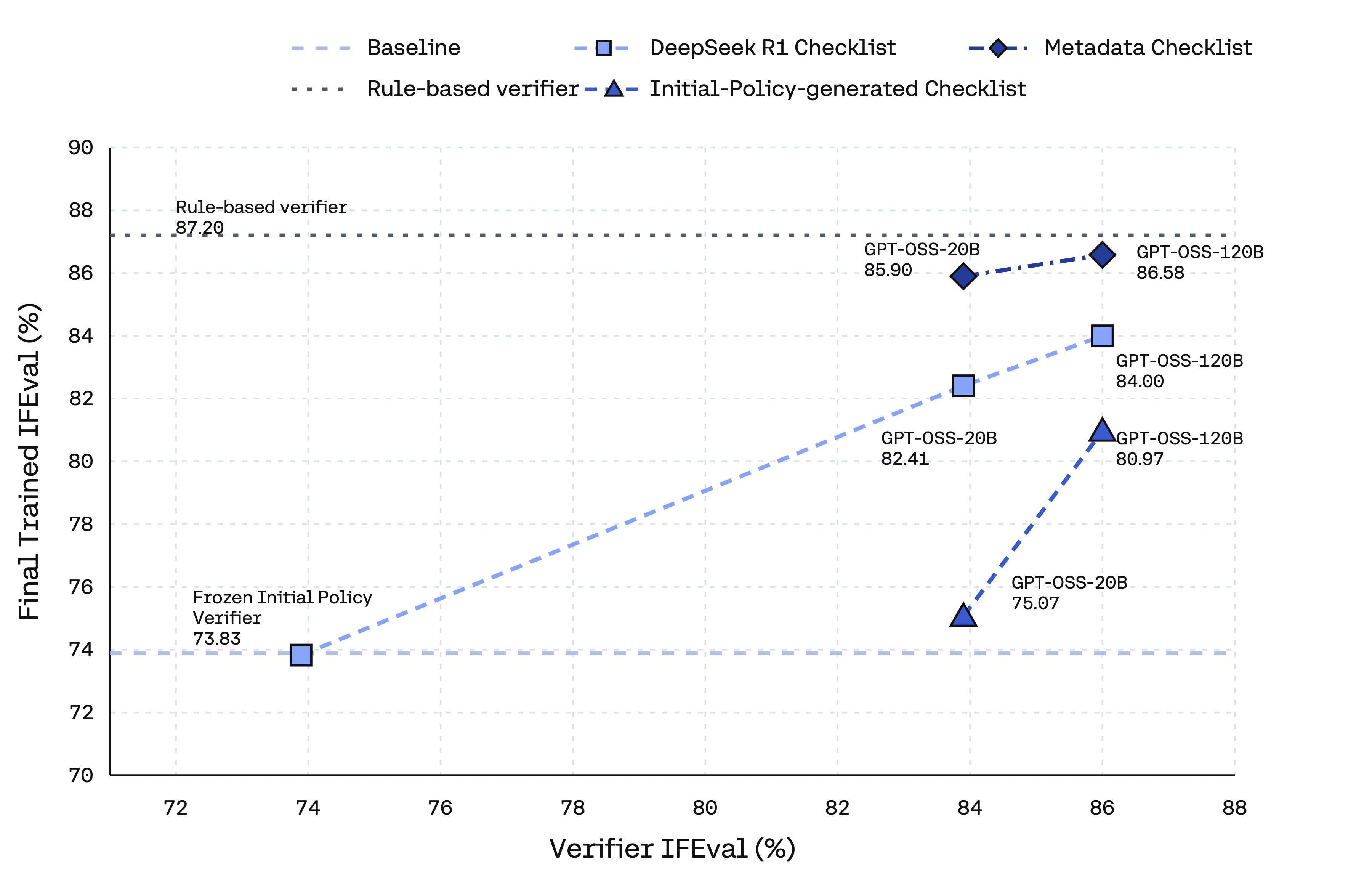}
    \caption{\textbf{Effect of verifier and checklist quality on downstream \ifeval{} performance.} The x-axis shows the verifier model's \ifeval{} score, used as a proxy for verifier quality; the y-axis shows the trained generator's \ifeval{} score. Curves compare various checklist sources: Initial Policy checklists, DeepSeek R1 checklists, and metadata checklists using the true constraints in the training data. Horizontal reference lines mark the pre-training baseline and the deterministic \ifeval{} rule-based verifier reference, which removes both checklist error and verifier noise. 
    }
    \label{fig:checklist_verifier_quality_line}
\end{figure}

Figure~\ref{fig:checklist_verifier_quality_line} shows that both checklist quality and verifier quality affect downstream RL performance. Holding the DeepSeek R1 checklist source fixed, stronger verifiers produce better trained generators. Using the initial policy itself as verifier yields essentially no improvement over the baseline: the trained model reaches 73.83\% \ifeval{} compared with the 73.89\% starting point. Replacing this verifier with GPT-OSS-20B raises \ifeval{} to 82.41\%, and GPT-OSS-120B raises it further to 84.00\%. This comparison isolates the effect of verifier quality under a fixed generated checklist.

Holding the verifier fixed isolates checklist quality. At both verifier scales, DeepSeek R1 checklists outperform Initial Policy checklists: 82.41\% vs. 75.07\% with GPT-OSS-20B (+7.34 points), and 84.00\% vs. 80.97\% with GPT-OSS-120B (+3.03 points). Using true training metadata constraints as checklists improves performance further, reaching 85.90\% with GPT-OSS-20B and 86.58\% with GPT-OSS-120B. These metadata checklists are not generally available, but they quantify headroom from better checklist generation. Our results are consistent with prior work showing that the quality of evaluation criteria can affect downstream optimization performance~\citep{gunjal2025rubrics}.

The deterministic rule-based verifier reaches 87.20\%, providing a reference point with both the ideal checklist and a noise-free verifier. The remaining gap between the metadata-checklist runs and this oracle run estimates the residual cost of verifier imperfection after checklist error is removed. Together, these results show that effective checklist-based soft RL requires both a sufficiently accurate verifier and well-specified checklist items.

\subsection{Self-Verification}
\label{sec:self-verification-results}

We next evaluate \softsverl{}, where the generator and verifier share parameters. This setting removes the need for a separate online reward model, but it also introduces the collapse mode described in Section~\ref{sec:self-verification}: the shared verifier can become more permissive on the generator's own samples, producing inflated rewards instead of better responses. We therefore systematically ablate the stabilization components in the shared objective from Eq.~\ref{eq:shared_objective}. All variants in this section are trained for 100 training steps.

\begin{figure}[!h]
    \centering
    \includegraphics[width=0.8\textwidth]{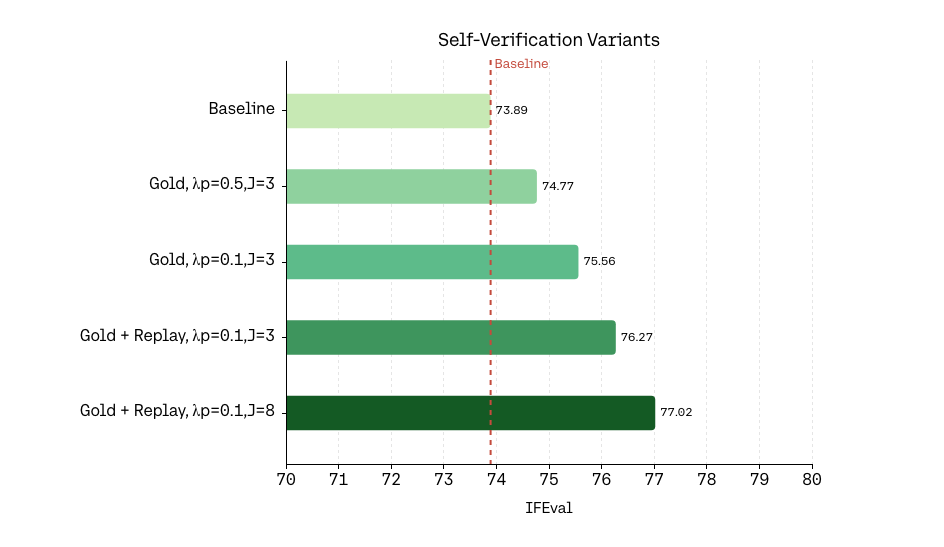}
    \caption{\textbf{Self-verification ablations on \ifeval{}.}
    Self-verification benefits from verifier co-training, replay data, multiple verifier votes, and a tuned partition penalty. Setting the partition weight too high hinder performance.}
    \label{fig:self-verification-variants}
\end{figure}

Figure~\ref{fig:self-verification-variants} shows that the stabilization components are important for effective self-verification.
For replay admission, we use thresholds $\tau_+=0.75$ and $\tau_-=0.375$.
With gold verifier-improvement examples, $J=3$ verifier votes per checklist item, and $\lambda_p=0.5$, \softsverl{} reaches 74.77\% \ifeval{}, a modest gain over the 73.89\% baseline.
Reducing the partition weight to $\lambda_p=0.1$ raises the final \ifeval{} to 75.56\%, indicating that the relative strengths of $\mathcal{J}_{\mathrm{ver}}$ and $\mathcal{J}_{\mathrm{part}}$ matter in shared-parameter training.
A partition penalty that is too large may overcorrect the verifier, suppressing positive judgments needed for useful reward.

Adding on-policy replay gives a larger improvement.
With the same $J=3$ vote count, including tagged examples from the current run increases \ifeval{} to 76.27\%.
This supports the role of $\mathcal{D}_{\mathrm{Replay}}$: gold examples anchor the verifier's notion of checklist satisfaction, while on-policy examples expose the verifier to the current generator's actual failure modes.
Increasing the verifier vote count from $J=3$ to $J=8$ further improves the gold-plus-replay setting from 76.27\% to 77.02\%, consistent with reduced sampling noise in the empirical vote fraction $\hat{p}_{ik}$.
We do not interpret this limited range as a monotonic scaling law.
Overall, self-verification benefits from the full \softsverl{} stabilization recipe: verifier co-training, on-policy replay, multiple verifier votes, and a tuned anti-collapse penalty.

\subsection{Checklist vs.\ Holistic Verification}
\label{sec:checklist-vs-holistic}

The analysis in Section~\ref{sec:checklist-mse} predicts that checklist verification should help when decomposition reduces verifier noise without introducing too much partial-credit bias. In particular, checklist rewards are expected to be most useful when item-level judgments are more sensitive than holistic judgments and when item-level specificity is not substantially worse. We test this prediction by comparing checklist-based verification against holistic verification, where the verifier evaluates the full response against the entire prompt in a single judgment.

To estimate the verifier quantities in the theory, we generate completions from the initial policy and grade the same completions with each GPT-OSS verifier. We use the \ifeval{} verifier to obtain holistic ground-truth labels for estimating single-judgment sensitivity and specificity, $p$ and $q$. For checklist estimates, we use the \ifeval{} constraints as checklist items and compare each LLM verifier's item-level decisions against the corresponding constraint labels, yielding $p'$ and $q'$.

\begin{table}[!h]
\centering
\caption{\textbf{Estimated holistic and checklist verifier quality on initial-policy completions.}
Checklist estimates use \ifeval{} constraints as item-level ground truth. ``Bias condition'' reports the fraction of samples for which the partial-credit bias is small enough for checklist verification to satisfy the sufficient condition in Proposition~\ref{prop:checklist-bias}.}
\label{tab:verifier-quality-estimates}
\small
\setlength{\tabcolsep}{4pt}
\begin{tabular}{lccccccccc}
\toprule
Verifier & $p$ & $p'$ & $q$ & $q'$ & $p'-p$ & $q'-q$ & $\alpha$ & $\alpha'$ & Bias cond. (\%) \\
\midrule
GPT-OSS-20B  & 0.477 & 0.873 & 0.662 & 0.717 & +0.396 & +0.055 & 0.139 & 0.590 & 81.5 \\
GPT-OSS-120B & 0.552 & 0.860 & 0.709 & 0.712 & +0.308 & +0.003 & 0.261 & 0.572 & 81.5 \\
\bottomrule
\end{tabular}
\end{table}

Table~\ref{tab:verifier-quality-estimates} shows that both checklist verifiers have much higher item-level sensitivity than holistic sensitivity (here strictly, $p' > p$), while item-level specificity is similar to or slightly higher than holistic specificity ($q' \ge q$). 
This matches the theory's main prediction: checklist decomposition is most helpful when it improves sensitivity and reduces variance without substantially weakening specificity.

\begin{figure}[!h]
    \centering
    \begin{minipage}[t]{0.5\textwidth}
        \centering
        \includegraphics[width=\textwidth,trim=60 20 110 25,clip]{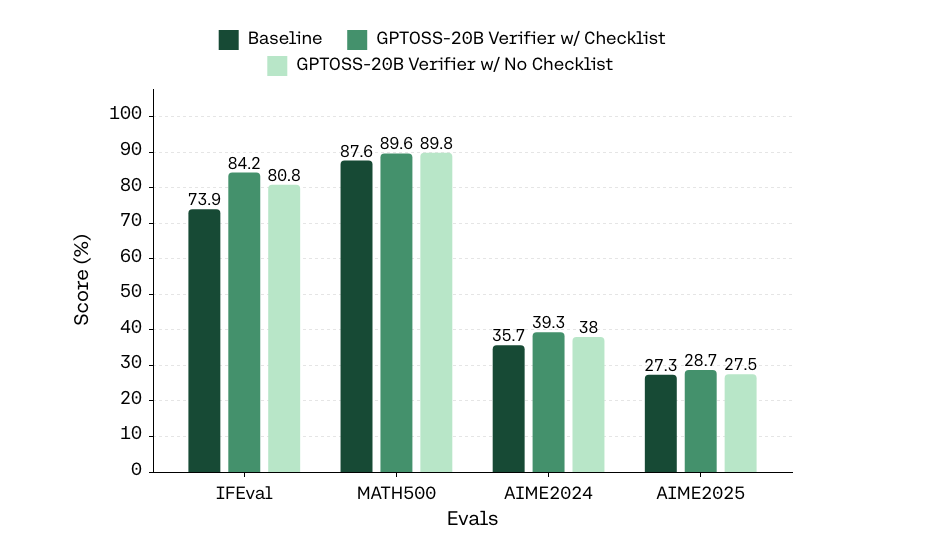}
        \vspace{-0.3em}
        \caption*{(a) GPT-OSS-20B verifier}
    \end{minipage}%
    \begin{minipage}[t]{0.5\textwidth}
        \centering
        \includegraphics[width=\textwidth,trim=110 20 35 25,clip]{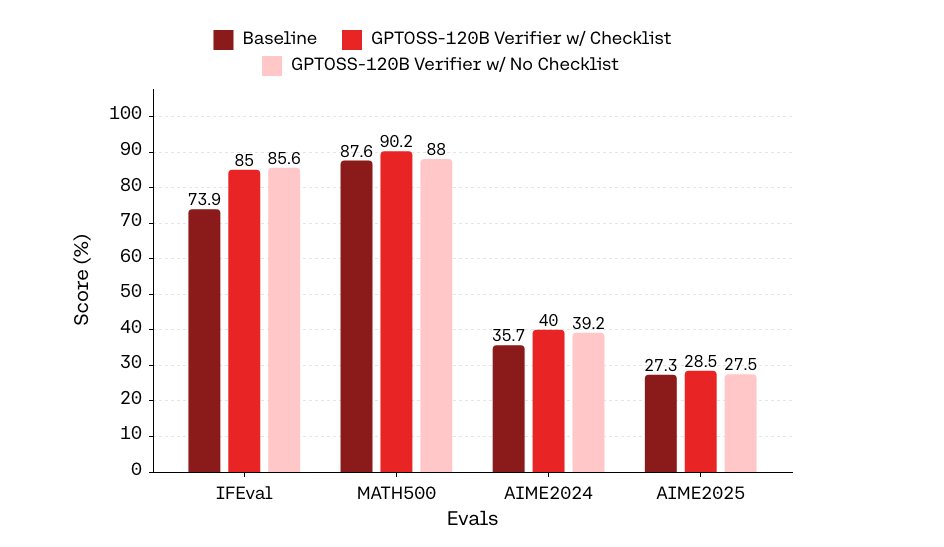}
        \vspace{-0.3em}
        \caption*{(b) GPT-OSS-120B verifier}
    \end{minipage}
    \caption{\textbf{Checklist-based vs.\ holistic (non-checklist) verification for two verifier scales.} 
    With the noisier GPT-OSS-20B verifier, checklist decomposition improves \ifeval{} substantially. With the stronger GPT-OSS-120B verifier, the \ifeval{} gap closes, though checklist-based training retains an edge on math benchmarks.}
    \label{fig:checklist_vs_nonchecklist}
\end{figure}

Figure~\ref{fig:checklist_vs_nonchecklist} compares checklist-based and holistic training at two verifier scales. With GPT-OSS-20B, checklist-based training outperforms holistic training by 3.4 points on \ifeval{} (84.20\% vs.\ 80.80\%), consistent with the prediction that decomposition helps most when verifier judgments are noisy. With GPT-OSS-120B, holistic training slightly exceeds checklist-based training on \ifeval{}, though the gap is small (85.58\% vs.\ 85.00\%), suggesting that stronger holistic judgments reduce the variance-reduction benefit of decomposition.

This result does not contradict the low holistic sensitivity in Table~\ref{tab:verifier-quality-estimates}, because that estimate measures agreement with the deterministic \ifeval{} rule-based label. Rule-based checks capture operational constraints, while holistic LLM verifiers may also judge the semantic content of the response. For example, consider the prompt: ``List 2 science fiction books with their authors. End the response with Happy Reading!'' The response ``$*$ 1984 by George Orwell; $*$ Brave New World by Aldous Huxley; Happy Reading!'' may satisfy the rule-based checks for two listed items and the required closing phrase, but a holistic verifier may reject it because these works are often classified more specifically as political fiction rather than as science fiction. Such disagreements count as false negatives relative to the rule-based label and lower the estimated holistic sensitivity, even though the holistic reward can still provide a useful training signal.

Checklist-based training still performs better on math benchmarks at both verifier scales. With GPT-OSS-120B, it reaches 90.20\% on MATH500 versus 88.00\% for holistic verification, and 40.00\% on AIME 2024 versus 39.17\%. One possible explanation is that graded partial credit provides more informative reward shaping than a binary holistic judgment, especially for tasks where partial progress matters.

\section{Related Work}
\label{sec:related_work}

\textbf{RL beyond exact verifiers.}
Recent work extends reinforcement learning beyond standard verifiable domains by replacing formal verification with weaker supervision signals. One line of work removes explicit rule-based verifiers but still relies on gold final answers or reference outcomes, for example by rewarding answer likelihood or reference-conditioned reasoning quality~\citep{zhou2025verifree,yu2025rlpr,chen2024latro,tang2025beyond}. 
These methods are important contrasts, but they do not address our target setting because they still assume access to a correct target answer or reference outcome. In contrast, \softrlvr{} targets tasks where no single answer string is available and correctness must instead be assessed through prompt-specific criteria.

\textbf{Self-generated and label-light rewards.}
A second line of work reduces labeled supervision through intrinsic or self-generated signals, including TTRL, Genius, and Learning to Reason without External Rewards~\citep{zuo2025ttrl,xu2025genius,zhao2025intuitor}. TTRL is especially relevant because it trains from unlabeled prompts and self-generated rewards rather than gold answers. 
However, its reward is based on majority-vote pseudo-labels on answer-extractable reasoning benchmarks. Current evidence does not show that the same mechanism extends to open-ended or checklist-style evaluation, where responses can satisfy some requirements but not others. These methods are therefore closer to \softsverl{} in being label-light than in addressing the same verification problem.

\textbf{Rubrics, checklists, and learned judges.}
Prior work has used rubrics, checklists, or learned judges to evaluate open-ended model outputs~\citep{gunjal2025rubrics,jia2025writingzero,jayalath2025compute}. \citet{gunjal2025rubrics} is the closest checklist-style prior: it constructs rubrics for non-verifiable tasks and scores outputs with an external judge. Our focus is different. We use checklist decomposition to construct dense, interpretable RL rewards and analyze when these rewards provide a reliable training signal. This extends the principle behind RLVR beyond domains with hand-written verifiers: the model is still evaluated against explicit criteria, but the checker is an LLM verifier rather than a symbolic rule or unit test.

More broadly, rubric-based approaches are closest to what we call a rubric-bound setting, where the main challenge is specifying useful criteria for open-ended tasks. \softrlvr{} instead targets a verifier-bound setting: even when evaluation criteria can be written as checklist items, the central problem is how to obtain stable and useful reward from an imperfect verifier. 

\textbf{Self-verifying RL.}
Reasoning over Mathematical Objects~\citep{aggarwal2026mathematical} is close to our work on the verifier side because it studies strong LM judges, on-policy judge training, and hard-to-verify mathematical-object equivalence. However, its setting remains centered on externally trained judges for mathematically structured outputs. URPO~\citep{lu2025urpo} is also relevant because the generator and evaluator co-evolve within one training loop, though it remains grounded in preference data and verifiable reasoning tasks. In contrast, \softsverl{} combines checklist-conditioned soft rewards, repeated verifier voting, verifier co-training from gold and replay supervision, and explicit anti-collapse regularization for self-verification.

\begin{table}[!ht]
\centering
\setlength{\tabcolsep}{2.5pt}
\renewcommand{\arraystretch}{1.15}
\caption{Supervision assumptions of closely related RL methods. ``Gold answer'' denotes reliance on a known final answer or reference outcome, not item-level verifier labels or preference data; ``Fixed Judge/RM'' denotes use of a fixed external judge or separately trained reward model. Partial entries indicate that the property applies only to part of the method or evaluation setting.}
\label{tab:related-work-comparison}
\footnotesize
\scalebox{0.95}{%
\begin{tabular}{@{}p{0.31\textwidth}ccccc@{}}
\toprule
Method & \mbox{Gold Answer} & \mbox{Fixed Judge/RM} & \mbox{Open-Ended} & Rubric & \mbox{Self Verification} \\
\midrule
Answer-likelihood RL \citep{zhou2025verifree,yu2025rlpr,chen2024latro,tang2025beyond}
& Yes & No & Yes & No & No \\
TTRL \citep{zuo2025ttrl}
& No & No & No & No & No \\
INTUITOR / GENIUS \citep{zhao2025intuitor,xu2025genius}
& No & No & Yes & No & No \\
Writing-Zero \citep{jia2025writingzero}
& No & Yes & Yes & No & No \\
URPO \citep{lu2025urpo}
& Mixed & No & Yes & No & Yes \\
Compute as Teacher \citep{jayalath2025compute}
& No & Yes & Yes & Yes & No \\
Rubrics as Rewards \citep{gunjal2025rubrics}
& Yes & Yes & Yes & Yes & No \\
Reasoning over Mathematical Objects \citep{aggarwal2026mathematical}
& Yes & Yes & Yes & No & No \\
\midrule
\textbf{Soft-SVeRL (ours)}
& \textbf{No} & \textbf{No} & \textbf{Yes} & \textbf{Yes} & \textbf{Yes} \\
\bottomrule
\end{tabular}%
}
\end{table}

\section{Conclusion}

We showed that learned LLM verifiers producing graded, checklist-based rewards can serve as effective training rewards for instruction-following RL. Training Command R7B with checklist-based \softrlvr{} improves \ifeval{} by up to 11.1 points, with no degradation on out-of-distribution math benchmarks. Decomposing prompts into per-item checklist judgments improves over holistic verification when verifier judgments are noisier, consistent with our theoretical MSE analysis. Our ablations further show that both verifier quality and checklist quality matter, and that self-verification requires explicit safeguards: \softsverl{} benefits from verifier co-training, replay, and anti-collapse regularization.

Several limitations qualify these findings. Checklist decomposition assumes that prompt requirements can be expressed as discrete, atomic pass-fail items, which may not hold for more subjective or holistic tasks. Our evaluation is also limited to one base model scale and one training domain. For self-verification, we also assume access to a set of gold-labeled verifier training data as an anchor for the verifier.

Future work should extend soft verification to less controlled settings such as open-ended writing and multi-turn dialogue, test self-verification at larger model scales and longer training horizons, and improve checklist generation without relying on a strong external model. Overall, checklist-based soft verification offers a path toward broadening RLVR beyond exactly verifiable domains by turning partially verifiable instructions into dense, actionable reward signals.

\bibliography{iclr2025_conference}
\clearpage
\appendix
\renewcommand{\proofname}{\textbf{Proof}}

\section{Appendix}

\subsection{Training Hyperparameters}
\label{app:training-hparams}

All runs use Group Relative Policy Optimization \citep{grpo} on a FAX-based \citep{yoo2022fax} asynchronous RL training framework. We sweep learning rates and find that $5 \times 10^{-6}$ provides the best balance of stability and performance across our runs. Unless stated otherwise, we train for 500 steps with a batch size of 256, group size 8, and maximum sequence length of 16{,}384 tokens. We use asymmetric clipping bounds from DAPO \citep{dapo}, with $\epsilon_{\text{low}} = 0.1$ and $\epsilon_{\text{high}} = 0.2$, and no KL regularization. For the verifier models, we use the highest thinking budget setting up to 32k tokens.

\subsection{SFT vs.\ RL}
\label{app:sft-vs-rl}

A natural alternative to RL with learned verifiers is supervised fine-tuning (SFT) on high-quality outputs from a stronger model. We compare RL training with a 20B verifier against SFT on outputs generated by GPT-OSS-120B in the high reasoning setting, giving SFT an advantage in both model scale and generation quality.

\begin{figure}[!h]
    \centering
    \includegraphics[width=0.8\textwidth]{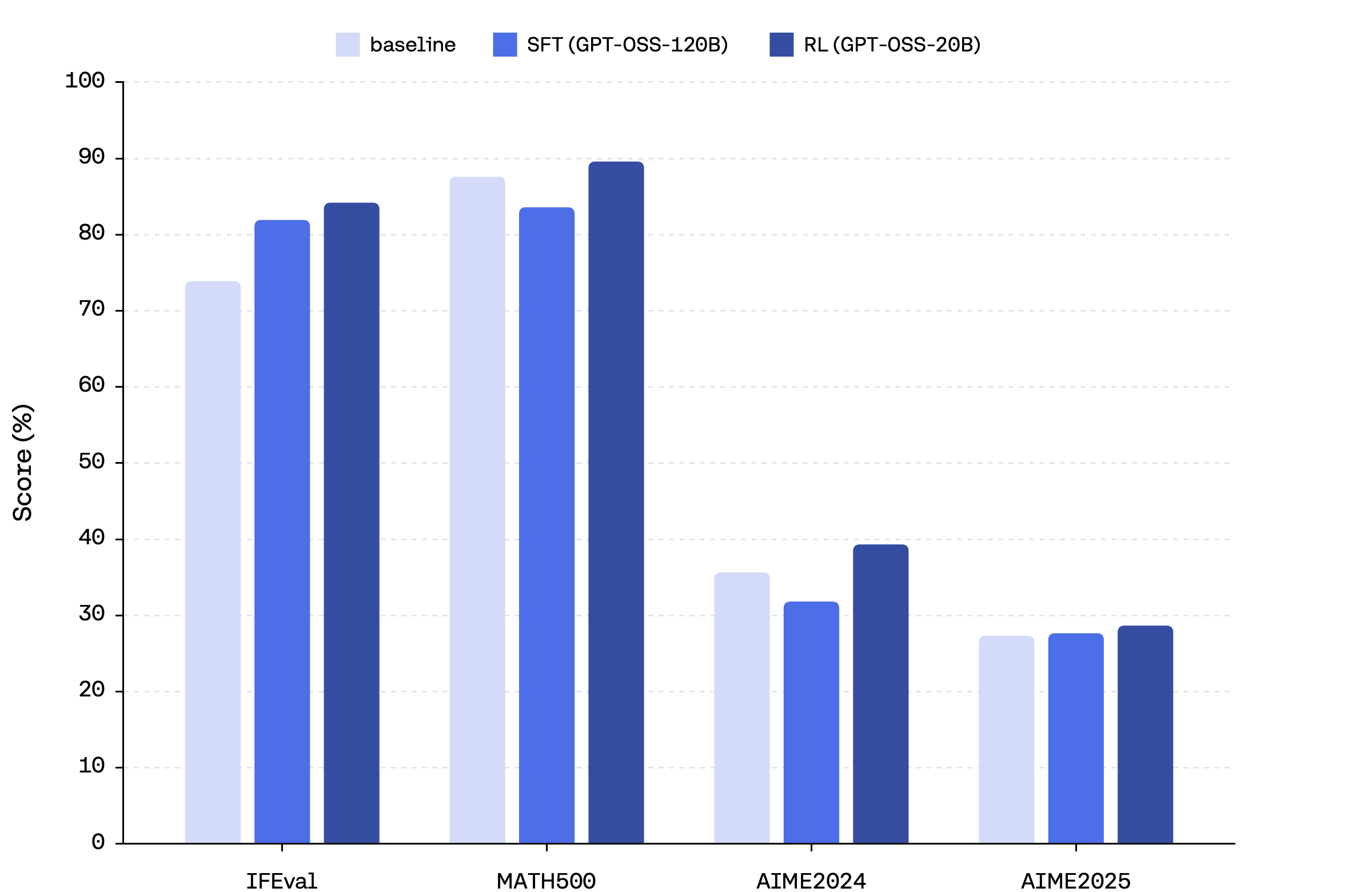}
    \caption{SFT on GPT-OSS-120B generated outputs vs.\ RL with GPT-OSS-20B verifier rewards. RL outperforms SFT on all benchmarks despite using a 6$\times$ smaller verifier model, and avoids the math degradation that SFT exhibits.}
    \label{fig:sft_vs_rl}
\end{figure}

Figure~\ref{fig:sft_vs_rl} shows that RL with the 20B verifier outperforms SFT with the 120B verifier on IFEval (84.20 vs.\ 81.93, +2.3 points) despite using a 6$\times$ smaller verifier model. The gap is more pronounced on math benchmarks. SFT degrades MATH500 from 87.60 to 83.60 ($-$4.0 points) and AIME 2024 from 35.67 to 31.83 ($-$3.8 points), while RL improves both (89.60 and 39.33, respectively).

The math degradation under SFT is consistent with the well-documented pattern of capability regression during supervised post-training: fine-tuning on domain-specific data can overwrite previously learned capabilities. RL avoids this because policy-gradient updates adjust the model toward higher-reward behavior without forcing it to match specific output distributions, preserving useful capabilities that are not penalized by the reward. This is consistent with recent evidence that RL training produces sparse, low-rank updates to model parameters, requiring far less capacity than supervised fine-tuning \citep{schulman2025lora}.

\subsection{Prompt Templates}
\label{app:prompt-templates}

\subsubsection{Verification Prompt Template}
\label{app:verification-prompt}

The following template is used to query the verifier model for each checklist item. The verifier receives the original instruction, the candidate response, and a single checklist question, and outputs a binary \texttt{yes}/\texttt{no} judgment.

\begin{verbatim}
Here is an instruction given to an AI Assistant:

<START OF INSTRUCTION>
{instruction}
<END OF INSTRUCTION>

Here is an AI Assistant's response to the instruction:

<START OF RESPONSE>
{response}
<END OF RESPONSE>

Please answer the following question about the AI
Assistant's response to the instruction:

<START OF QUESTION>
{checklist_question}
<END OF QUESTION>

Output "yes" or "no" and nothing else.
\end{verbatim}

\subsubsection{Checklist Generation Prompt Template}
\label{app:checklist-prompt}

The following template is used to generate checklist items from a prompt. The checklist generator produces an ordered list of yes/no questions that collectively cover the requirements of the instruction.

\begin{verbatim}
Here is an instruction given to an AI assistant:

<START OF INSTRUCTION>
{prompt}
<END OF INSTRUCTION>

Please write an exhaustive checklist of yes/no questions
that could be used to evaluate whether the AI assistant's
response to the question is completely correct. The answer
to each of the questions should be yes if and only if the
AI assistant's response is correct. Be sure to consider
edge cases if applicable. The checklist should not include
questions which are not absolutely essential for a
completely correct response even if the answer "yes" would
likely imply a better response. The ordering of the
checklist should be such that the questions are ordered
from the most important to the least important. Output the
checklist as a markdown numbered list. DO NOT output
anything else besides the numbered list.
\end{verbatim}

\subsubsection{Generated checklist length distribution}
\label{app:checklist-distribution}

Figure~\ref{fig:checklist-distribution} summarizes the number of checklist items produced by the LLM checklist generator across the training prompts. Most generated checklists contain between five and eight items, with a long tail of prompts receiving more detailed decompositions.

\begin{figure}[h!]
    \centering
    \includegraphics[width=0.75\textwidth]{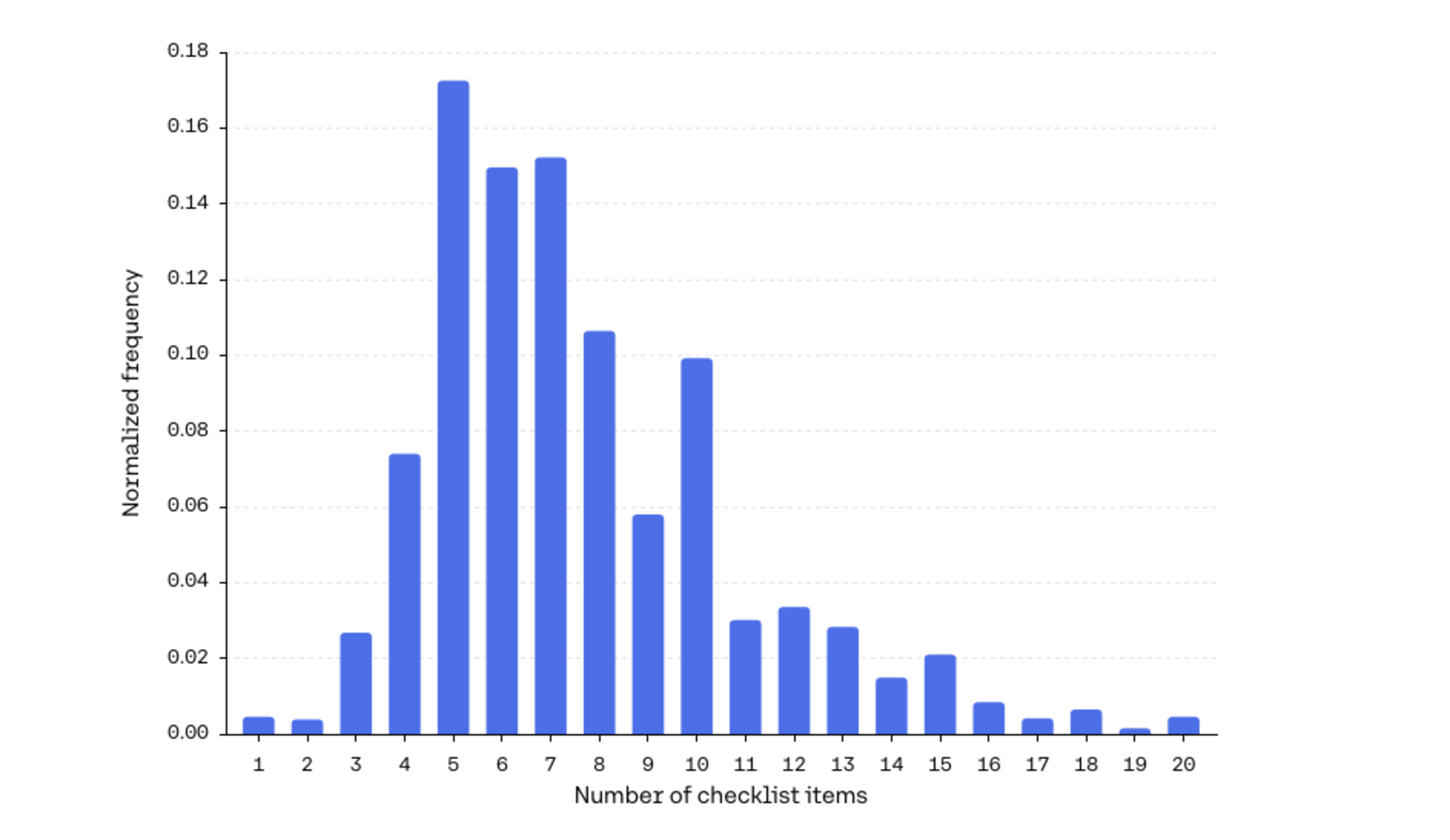}
    \caption{\textbf{Distribution of LLM-generated checklist lengths.} The x-axis gives the number of checklist items generated for a prompt, and the y-axis reports the normalized frequency.}
    \label{fig:checklist-distribution}
\end{figure}

\subsubsection{Checklist examples and failure modes}
\label{app:checklist-examples}

Table~\ref{tab:checklist-examples} shows representative checklist outputs. 
In general, useful checklist items are atomic, necessary, and directly verifiable from the instruction and candidate response. 
Common failure modes include under-decomposition, over-decomposition, spurious criteria, and missing constraints.

\begin{table}[h!]
\centering
\small
\begin{tabular}{p{0.28\linewidth}p{0.32\linewidth}p{0.32\linewidth}}
\toprule
Instruction & Good checklist items & Failure mode examples \\
\midrule
Write exactly three bullet points about Paris. Do not use the word ``travel''.
&
\begin{itemize}[leftmargin=*, nosep, topsep=-2pt]
    \item Does the response contain exactly three bullet points?
    \item Is each bullet point about Paris?
    \item Does the response avoid the word ``travel''?
\end{itemize}
&
\begin{itemize}[leftmargin=*, nosep, topsep=-2pt]
    \item Is the response engaging? \emph{(spurious)}
    \item Does the response mention Paris and use three bullet points?
    \emph{(under-decomposed)}
\end{itemize}
\\
\bottomrule
\end{tabular}
\caption{Illustrative checklist items and checklist-generation failure modes.}
\label{tab:checklist-examples}
\end{table}

\subsection{Additional Results and Proofs for Checklist MSE Analysis}
\label{app:additional-checklist-mse}

This section collects the auxiliary results used in Section~\ref{sec:checklist-mse}.
We use the notation from the main text: $\mathbf{Y}^*(x,y)$ denotes the latent checklist vector,
$S^*(x,y)$ denotes strict success, $\bar S^*(x,y)$ denotes partial credit, and
$\Delta^*(x,y)=\bar S^*(x,y)-S^*(x,y)$ denotes the relaxation gap.

\subsubsection{Proof of Proposition~\ref{prop:checklist-variance}}
\label{app:proof-prop-checklist-variance}

\begin{proof}
Under conditional independence, $\mathrm{Cov}(J_k,J_\ell\mid x,y)=0$ for all $k\ne \ell$.
The general variance formula from Proposition~\ref{prop:variance} reduces to
\[
\mathrm{Var}(\bar J\mid x,y)
=
\frac{1}{K^2}
\sum_{k=1}^K
\mu_k(x,y)\bigl(1-\mu_k(x,y)\bigr).
\]
Since $\mu_k(x,y)\in[0,1]$, we have $\mu_k(x,y)(1-\mu_k(x,y))\le 1/4$ for each $k$, so
\[
\mathrm{Var}(\bar J\mid x,y)
\le
\frac{1}{K^2}\cdot K\cdot \frac{1}{4}
=
\frac{1}{4K}.
\]
Multiplying by $\|s_\theta(x,y)\|_2^2$ gives
\[
\mathbb{E}\!\left[
\|g_{\mathrm{chk}}-\mathbb{E}[g_{\mathrm{chk}}\mid x,y]\|_2^2
\,\middle|\,x,y
\right]
\le
\frac{1}{4K}\|s_\theta(x,y)\|_2^2.
\]
\end{proof}

\subsubsection{Conditional bias and variance}
\label{app:prop-bias}

\begin{proposition}[Conditional bias of single-verifier and checklist estimators]
\label{prop:bias}
Under the setup of Section~\ref{sec:checklist-mse}, conditioning on $(x,y)$,
\[
\mathbb{E}[g_{\mathrm{single}}\mid x,y]-g^*(x,y)
=
\Bigl((1-q)+(\alpha-1)S^*(x,y)\Bigr)s_\theta(x,y),
\]
and
\[
\mathbb{E}[g_{\mathrm{chk}}\mid x,y]-g^*(x,y)
=
\Bigl((1-q')+\alpha'\Delta^*(x,y)+(\alpha'-1)S^*(x,y)\Bigr)s_\theta(x,y).
\]
Equivalently,
\[
\left\|\mathbb{E}[g_{\mathrm{single}}\mid x,y]-g^*(x,y)\right\|_2^2
=
\Bigl((1-q)+(\alpha-1)S^*(x,y)\Bigr)^2
\|s_\theta(x,y)\|_2^2,
\]
and
\[
\left\|\mathbb{E}[g_{\mathrm{chk}}\mid x,y]-g^*(x,y)\right\|_2^2
=
\Bigl((1-q')+\alpha'\Delta^*(x,y)+(\alpha'-1)S^*(x,y)\Bigr)^2
\|s_\theta(x,y)\|_2^2.
\]
\end{proposition}

\begin{proof}
Fix $(x,y)$; then $s_\theta(x,y)$ and $b(x)$ are deterministic.
For the single verifier,
\[
\mathbb{E}[g_{\mathrm{single}}\mid x,y]
=
\bigl(\mathbb{E}[J\mid x,y]-b(x)\bigr)s_\theta(x,y).
\]
Using
\[
\mathbb{E}[J\mid x,y]=(1-q)+\alpha S^*(x,y)
\]
and subtracting
\[
g^*(x,y)=\bigl(S^*(x,y)-b(x)\bigr)s_\theta(x,y)
\]
gives
\[
\mathbb{E}[g_{\mathrm{single}}\mid x,y]-g^*(x,y)
=
\Bigl((1-q)+\alpha S^*(x,y)-S^*(x,y)\Bigr)s_\theta(x,y)
=
\Bigl((1-q)+(\alpha-1)S^*(x,y)\Bigr)s_\theta(x,y).
\]

For the checklist estimator, linearity of expectation gives
\[
\mathbb{E}[\bar J\mid x,y]
=
\frac{1}{K}\sum_{k=1}^K
\bigl((1-q')+\alpha'Y_k^*(x,y)\bigr)
=
(1-q')+\alpha'\bar S^*(x,y).
\]
Subtracting $g^*(x,y)$ and using $\bar S^*(x,y)=S^*(x,y)+\Delta^*(x,y)$ yields
\[
\mathbb{E}[g_{\mathrm{chk}}\mid x,y]-g^*(x,y)
=
\Bigl((1-q')+\alpha'\bar S^*(x,y)-S^*(x,y)\Bigr)s_\theta(x,y)
\]
\[
=
\Bigl((1-q')+\alpha'\Delta^*(x,y)+(\alpha'-1)S^*(x,y)\Bigr)s_\theta(x,y).
\]
Taking squared norms gives the stated formulas.
\end{proof}

\subsubsection{Conditional variance}
\label{app:prop-variance}

\begin{proposition}[Conditional variance of single-verifier and checklist estimators]
\label{prop:variance}
Under the setup of Section~\ref{sec:checklist-mse}, conditioning on $(x,y)$,
\[
\mathrm{Cov}(g_{\mathrm{single}}\mid x,y)
=
\mathrm{Var}(J\mid x,y)\,s_\theta(x,y)s_\theta(x,y)^\top,
\]
and
\[
\mathrm{Cov}(g_{\mathrm{chk}}\mid x,y)
=
\mathrm{Var}(\bar J\mid x,y)\,s_\theta(x,y)s_\theta(x,y)^\top,
\]
where
\[
\mathrm{Var}(J\mid x,y)
=
\mu_{\mathrm{single}}(x,y)\bigl(1-\mu_{\mathrm{single}}(x,y)\bigr),
\qquad
\mu_{\mathrm{single}}(x,y):=(1-q)+\alpha S^*(x,y),
\]
and
\[
\mathrm{Var}(\bar J\mid x,y)
=
\frac{1}{K^2}
\left(
\sum_{k=1}^K \mu_k(x,y)\bigl(1-\mu_k(x,y)\bigr)
+
2\sum_{1\le k<\ell\le K}
\mathrm{Cov}(J_k,J_\ell\mid x,y)
\right),
\]
with
\[
\mu_k(x,y):=(1-q')+\alpha'Y_k^*(x,y).
\]
Equivalently,
\[
\mathbb{E}\!\left[
\|g_{\mathrm{single}}-\mathbb{E}[g_{\mathrm{single}}\mid x,y]\|_2^2
\,\middle|\,x,y
\right]
=
\|s_\theta(x,y)\|_2^2\,
\mu_{\mathrm{single}}(x,y)\bigl(1-\mu_{\mathrm{single}}(x,y)\bigr),
\]
and
\[
\mathbb{E}\!\left[
\|g_{\mathrm{chk}}-\mathbb{E}[g_{\mathrm{chk}}\mid x,y]\|_2^2
\,\middle|\,x,y
\right]
=
\|s_\theta(x,y)\|_2^2\,
\mathrm{Var}(\bar J\mid x,y).
\]
Under the additional assumption that $J_1,\dots,J_K$ are conditionally independent given $(x,y)$,
\[
\mathrm{Var}(\bar J\mid x,y)
=
\frac{1}{K^2}
\sum_{k=1}^K
\mu_k(x,y)\bigl(1-\mu_k(x,y)\bigr)
\le
\frac{1}{4K}.
\]
\end{proposition}

\begin{proof}
Condition on $(x,y)$; then $s_\theta(x,y)$ and $b(x)$ are deterministic and all randomness comes from verifier outputs.
Since
\[
g_{\mathrm{single}}=(J-b(x))s_\theta(x,y),
\]
we have
\[
g_{\mathrm{single}}-\mathbb{E}[g_{\mathrm{single}}\mid x,y]
=
\bigl(J-\mathbb{E}[J\mid x,y]\bigr)s_\theta(x,y).
\]
Therefore,
\[
\mathrm{Cov}(g_{\mathrm{single}}\mid x,y)
=
\mathrm{Var}(J\mid x,y)\,s_\theta(x,y)s_\theta(x,y)^\top.
\]
Since $J\in\{0,1\}$ is Bernoulli with mean
\[
\mu_{\mathrm{single}}(x,y)=(1-q)+\alpha S^*(x,y),
\]
its variance is
\[
\mu_{\mathrm{single}}(x,y)\bigl(1-\mu_{\mathrm{single}}(x,y)\bigr).
\]

The same argument applies to
\[
g_{\mathrm{chk}}=(\bar J-b(x))s_\theta(x,y),
\]
giving
\[
\mathrm{Cov}(g_{\mathrm{chk}}\mid x,y)
=
\mathrm{Var}(\bar J\mid x,y)\,s_\theta(x,y)s_\theta(x,y)^\top.
\]
Expanding the variance of the average,
\[
\mathrm{Var}(\bar J\mid x,y)
=
\frac{1}{K^2}
\left(
\sum_{k=1}^K \mathrm{Var}(J_k\mid x,y)
+
2\sum_{1\le k<\ell\le K}\mathrm{Cov}(J_k,J_\ell\mid x,y)
\right).
\]
Since each $J_k$ is Bernoulli with mean
\[
\mu_k(x,y)=(1-q')+\alpha'Y_k^*(x,y),
\]
we have
\[
\mathrm{Var}(J_k\mid x,y)=\mu_k(x,y)\bigl(1-\mu_k(x,y)\bigr).
\]
The norm-squared identities follow from
\[
\mathbb{E}\!\left[\|(Z-\mathbb{E}[Z])s\|_2^2\right]
=
\|s\|_2^2\mathrm{Var}(Z).
\]
Under conditional independence, the cross-covariance terms vanish, and
$\mu_k(1-\mu_k)\le 1/4$ gives the $1/(4K)$ bound.
\end{proof}

\subsubsection{Bias-side condition}
\label{app:proof-prop-checklist-bias}

\begin{proposition}[Sufficient conditions for gradient bias reduction]
\label{prop:checklist-bias}
Under the setup of Section~\ref{sec:checklist-mse},
\[
\left\|
\mathbb{E}[g_{\mathrm{chk}}\mid x,y]-g^*(x,y)
\right\|_2^2
\le
\left\|
\mathbb{E}[g_{\mathrm{single}}\mid x,y]-g^*(x,y)
\right\|_2^2
\]
whenever either
\[
S^*(x,y)=1
\qquad\text{and}\qquad
p'\ge p,
\]
or
\[
S^*(x,y)=0
\qquad\text{and}\qquad
\alpha'\Delta^*(x,y)\le q'-q.
\]
\end{proposition}

\begin{proof}
By Proposition~\ref{prop:bias},
\[
\mathbb{E}[g_{\mathrm{single}}\mid x,y]-g^*(x,y)
=
\Bigl((1-q)+(\alpha-1)S^*(x,y)\Bigr)s_\theta(x,y),
\]
and
\[
\mathbb{E}[g_{\mathrm{chk}}\mid x,y]-g^*(x,y)
=
\Bigl((1-q')+\alpha'\Delta^*(x,y)+(\alpha'-1)S^*(x,y)\Bigr)s_\theta(x,y).
\]

\textbf{\textit{Case (i): $S^*(x,y)=1$.}}
When $S^*(x,y)=1$, all checklist items are satisfied, so $\bar S^*(x,y)=1$ and
$\Delta^*(x,y)=0$.
The bias expressions reduce to
\[
\left\|
\mathbb{E}[g_{\mathrm{single}}\mid x,y]-g^*(x,y)
\right\|_2^2
=
(1-p)^2\|s_\theta(x,y)\|_2^2,
\]
and
\[
\left\|
\mathbb{E}[g_{\mathrm{chk}}\mid x,y]-g^*(x,y)
\right\|_2^2
=
(1-p')^2\|s_\theta(x,y)\|_2^2.
\]
Since $p'\ge p$ implies $(1-p')^2\le (1-p)^2$, the checklist squared bias is no larger.

\textbf{\textit{Case (ii): $S^*(x,y)=0$.}}
When $S^*(x,y)=0$, the bias expressions reduce to
\[
\left\|
\mathbb{E}[g_{\mathrm{single}}\mid x,y]-g^*(x,y)
\right\|_2^2
=
(1-q)^2\|s_\theta(x,y)\|_2^2,
\]
and
\[
\left\|
\mathbb{E}[g_{\mathrm{chk}}\mid x,y]-g^*(x,y)
\right\|_2^2
=
\bigl(1-q'+\alpha'\Delta^*(x,y)\bigr)^2
\|s_\theta(x,y)\|_2^2.
\]
If $\alpha'\Delta^*(x,y)\le q'-q$, then
\[
1-q'+\alpha'\Delta^*(x,y)\le 1-q,
\]
so the checklist squared bias is no larger.
\end{proof}

\subsubsection{Proof of Theorem~\ref{thm:checklist-mse}}
\label{app:proof-thm-checklist-mse}

\begin{proposition}[Bias--variance decomposition of gradient MSE]
\label{prop:mse-decomp}
Under the setup of Section~\ref{sec:checklist-mse}, conditioning on $(x,y)$,
\[
\mathbb{E}\!\left[
\|g_{\mathrm{single}}-g^*\|_2^2
\,\middle|\,x,y
\right]
=
\left\|
\mathbb{E}[g_{\mathrm{single}}\mid x,y]-g^*(x,y)
\right\|_2^2
+
\mathbb{E}\!\left[
\|g_{\mathrm{single}}-\mathbb{E}[g_{\mathrm{single}}\mid x,y]\|_2^2
\,\middle|\,x,y
\right],
\]
and
\[
\mathbb{E}\!\left[
\|g_{\mathrm{chk}}-g^*\|_2^2
\,\middle|\,x,y
\right]
=
\left\|
\mathbb{E}[g_{\mathrm{chk}}\mid x,y]-g^*(x,y)
\right\|_2^2
+
\mathbb{E}\!\left[
\|g_{\mathrm{chk}}-\mathbb{E}[g_{\mathrm{chk}}\mid x,y]\|_2^2
\,\middle|\,x,y
\right].
\]
\end{proposition}

\begin{proof}
For any random vector $G$ and deterministic vector $a$,
\[
\mathbb{E}\!\left[\|G-a\|_2^2\right]
=
\|\mathbb{E}[G]-a\|_2^2
+
\mathbb{E}\!\left[\|G-\mathbb{E}[G]\|_2^2\right].
\]
Apply this identity conditionally on $(x,y)$ with
$G=g_{\mathrm{single}}$ and $a=g^*(x,y)$, and then with
$G=g_{\mathrm{chk}}$ and $a=g^*(x,y)$.
\end{proof}

\begin{proof}[Proof of Theorem~\ref{thm:checklist-mse}]
We use the bias--variance decomposition in Proposition~\ref{prop:mse-decomp} and analyze the two cases separately.

\textbf{\textit{Case (i): $S^*(x,y)=1$.}}
When $S^*(x,y)=1$, the relaxation gap $\Delta^*(x,y)=0$.
By Proposition~\ref{prop:bias}, the squared biases are
\[
\left\|
\mathbb{E}[g_{\mathrm{single}}\mid x,y]-g^*(x,y)
\right\|_2^2
=
(1-p)^2\|s_\theta(x,y)\|_2^2,
\]
and
\[
\left\|
\mathbb{E}[g_{\mathrm{chk}}\mid x,y]-g^*(x,y)
\right\|_2^2
=
(1-p')^2\|s_\theta(x,y)\|_2^2.
\]
For the variance terms, when $S^*(x,y)=1$, the conditional means are
$\mu_{\mathrm{single}}(x,y)=p$ and $\mu_k(x,y)=p'$ for all $k$.
Thus,
\[
\mathbb{E}\!\left[
\|g_{\mathrm{single}}-\mathbb{E}[g_{\mathrm{single}}\mid x,y]\|_2^2
\,\middle|\,x,y
\right]
=
p(1-p)\|s_\theta(x,y)\|_2^2,
\]
while conditional independence gives
\[
\mathbb{E}\!\left[
\|g_{\mathrm{chk}}-\mathbb{E}[g_{\mathrm{chk}}\mid x,y]\|_2^2
\,\middle|\,x,y
\right]
=
\frac{p'(1-p')}{K}\|s_\theta(x,y)\|_2^2.
\]
Therefore,
\[
\mathbb{E}\!\left[
\|g_{\mathrm{chk}}-g^*\|_2^2
\,\middle|\,x,y
\right]
=
\left[
(1-p')^2+\frac{p'(1-p')}{K}
\right]\|s_\theta(x,y)\|_2^2.
\]
Since $K\ge 1$,
\[
(1-p')^2+\frac{p'(1-p')}{K}
\le
(1-p')^2+p'(1-p')
=
1-p'.
\]
If $p'\ge p$, then $1-p'\le 1-p$.
Also,
\[
(1-p)^2+p(1-p)=1-p.
\]
Hence
\[
\mathbb{E}\!\left[
\|g_{\mathrm{chk}}-g^*\|_2^2
\,\middle|\,x,y
\right]
\le
(1-p)\|s_\theta(x,y)\|_2^2
=
\mathbb{E}\!\left[
\|g_{\mathrm{single}}-g^*\|_2^2
\,\middle|\,x,y
\right].
\]

\textbf{\textit{Case (ii): $S^*(x,y)=0$.}}
When $S^*(x,y)=0$, Proposition~\ref{prop:bias} gives
\[
\left\|
\mathbb{E}[g_{\mathrm{single}}\mid x,y]-g^*(x,y)
\right\|_2^2
=
(1-q)^2\|s_\theta(x,y)\|_2^2,
\]
and
\[
\left\|
\mathbb{E}[g_{\mathrm{chk}}\mid x,y]-g^*(x,y)
\right\|_2^2
=
\bigl(1-q'+\alpha'\Delta^*(x,y)\bigr)^2
\|s_\theta(x,y)\|_2^2.
\]
The single-verifier variance term is
\[
(1-q)q\,\|s_\theta(x,y)\|_2^2,
\]
and the checklist variance term is bounded by
\[
\frac{1}{4K}\|s_\theta(x,y)\|_2^2
\]
by Proposition~\ref{prop:checklist-variance}.
Thus,
\[
\mathbb{E}\!\left[
\|g_{\mathrm{single}}-g^*\|_2^2
\,\middle|\,x,y
\right]
=
\bigl[(1-q)^2+(1-q)q\bigr]\|s_\theta(x,y)\|_2^2
=
(1-q)\|s_\theta(x,y)\|_2^2.
\]
A sufficient condition for the checklist MSE to be no larger is
\[
\bigl(1-q'+\alpha'\Delta^*(x,y)\bigr)^2+\frac{1}{4K}\le 1-q,
\]
which is exactly the condition in the theorem.
\end{proof}

\subsubsection{Checklist-size implications}

\begin{corollary}[Bias condition implies MSE condition for large $K$]
\label{cor:A-implies-B}
In the regime $S^*(x,y)=0$, let
\[
\text{(A)}\quad \alpha'\Delta^*(x,y)\le q'-q
\qquad\text{and}\qquad
\text{(B)}\quad
\bigl(1-q'+\alpha'\Delta^*(x,y)\bigr)^2+\frac{1}{4K}\le 1-q.
\]
Condition~(A) is the bias-only sufficient condition of Proposition~\ref{prop:checklist-bias};
condition~(B) is the sufficient condition for full MSE improvement in
Theorem~\ref{thm:checklist-mse}.
If~(A) holds and
\[
K\ge \frac{1}{4(1-q)q},
\]
then~(B) holds, and therefore checklist verification achieves lower conditional MSE than the
single-verifier estimator.
\end{corollary}

\begin{proof}
Assume condition~(A), i.e.,
\[
\alpha'\Delta^*(x,y)\le q'-q.
\]
Then
\[
1-q'+\alpha'\Delta^*(x,y)\le 1-q,
\]
so
\[
\bigl(1-q'+\alpha'\Delta^*(x,y)\bigr)^2\le (1-q)^2.
\]
It therefore suffices to show
\[
(1-q)^2+\frac{1}{4K}\le 1-q,
\]
which is equivalent to
\[
\frac{1}{4K}\le (1-q)q.
\]
This holds whenever
\[
K\ge \frac{1}{4(1-q)q}.
\]
Hence condition~(B) holds, and Theorem~\ref{thm:checklist-mse} gives the MSE inequality.
\end{proof}

\begin{corollary}[Checklist-size thresholds]
\label{cor:size-threshold}
On incorrect outputs $(S^*(x,y)=0)$, let
\[
A(x,y):=1-q'+\alpha'\Delta^*(x,y).
\]
If $A(x,y)\le 1-q$ and
\[
K\ge \frac{1}{4A(x,y)\bigl(1-A(x,y)\bigr)},
\]
then
\[
\mathbb{E}\!\left[
\|g_{\mathrm{chk}}-g^*\|_2^2
\,\middle|\,x,y
\right]
\le
\mathbb{E}\!\left[
\|g_{\mathrm{single}}-g^*\|_2^2
\,\middle|\,x,y
\right].
\]

On correct outputs $(S^*(x,y)=1)$, if $(1-p')^2<1-p$ and
\[
K\ge
\frac{p'(1-p')}{(1-p)-(1-p')^2},
\]
then
\[
\mathbb{E}\!\left[
\|g_{\mathrm{chk}}-g^*\|_2^2
\,\middle|\,x,y
\right]
\le
\mathbb{E}\!\left[
\|g_{\mathrm{single}}-g^*\|_2^2
\,\middle|\,x,y
\right].
\]
\end{corollary}

\begin{proof}\leavevmode\par
\noindent\textbf{\textit{Incorrect outputs.}}
Since $A(x,y)\le 1-q$, it suffices to show
\[
A(x,y)^2+\frac{1}{4K}\le A(x,y).
\]
This is equivalent to
\[
\frac{1}{4K}\le A(x,y)\bigl(1-A(x,y)\bigr),
\]
which is exactly the assumed lower bound on $K$.
Under this condition,
\[
A(x,y)^2+\frac{1}{4K}
\le
A(x,y)
\le
1-q,
\]
so the incorrect-output condition in Theorem~\ref{thm:checklist-mse} holds.

\noindent\textbf{\textit{Correct outputs.}}
From the proof of Theorem~\ref{thm:checklist-mse}, the checklist MSE normalized by
$\|s_\theta(x,y)\|_2^2$ is
\[
(1-p')^2+\frac{p'(1-p')}{K},
\]
and the single-verifier MSE normalized by $\|s_\theta(x,y)\|_2^2$ is
\[
(1-p)^2+p(1-p)=1-p.
\]
The checklist MSE is no larger whenever
\[
(1-p')^2+\frac{p'(1-p')}{K}\le 1-p.
\]
Since $(1-p')^2<1-p$ by assumption, rearranging gives
\[
K\ge
\frac{p'(1-p')}{(1-p)-(1-p')^2},
\]
which is the stated threshold.
\end{proof}

\begin{corollary}[Necessity of a bias-side condition]
\label{cor:bias-necessary}
Larger $K$ reduces variance but does not affect the structural bias term, so no condition on $K$ alone can guarantee checklist MSE improvement.

On incorrect outputs, if
\[
\bigl(1-q'+\alpha'\Delta^*(x,y)\bigr)^2 > 1-q,
\]
then
\[
\mathbb{E}\!\left[
\|g_{\mathrm{chk}}-g^*\|_2^2
\,\middle|\,x,y
\right]
>
\mathbb{E}\!\left[
\|g_{\mathrm{single}}-g^*\|_2^2
\,\middle|\,x,y
\right]
\]
for all finite $K$.

Similarly, on correct outputs, if
\[
(1-p')^2 > 1-p,
\]
then the checklist estimator has larger conditional MSE than the single-verifier estimator for all finite $K$.
\end{corollary}

\begin{proof}\leavevmode\par
\noindent\textbf{\textit{Incorrect outputs.}}
The checklist MSE is at least its squared bias:
\[
\mathbb{E}\!\left[
\|g_{\mathrm{chk}}-g^*\|_2^2
\,\middle|\,x,y
\right]
\ge
\bigl(1-q'+\alpha'\Delta^*(x,y)\bigr)^2
\|s_\theta(x,y)\|_2^2.
\]
The single-verifier MSE on incorrect outputs is
\[
(1-q)\|s_\theta(x,y)\|_2^2.
\]
Therefore, if
\[
\bigl(1-q'+\alpha'\Delta^*(x,y)\bigr)^2>1-q,
\]
then the checklist MSE is strictly larger for all finite $K$.

\noindent\textbf{\textit{Correct outputs.}}
The checklist MSE on correct outputs is
\[
\left[
(1-p')^2+\frac{p'(1-p')}{K}
\right]\|s_\theta(x,y)\|_2^2,
\]
while the single-verifier MSE is
\[
(1-p)\|s_\theta(x,y)\|_2^2.
\]
If $(1-p')^2>1-p$, then
\[
(1-p')^2+\frac{p'(1-p')}{K}>1-p
\]
for all finite $K$, so the checklist MSE is strictly larger.
\end{proof}

Together with Corollary~\ref{cor:A-implies-B}, Corollary~\ref{cor:bias-necessary} gives the main checklist-size message: once the bias-side condition is controlled, sufficiently large $K$ can make checklist verification lower-MSE than holistic verification; but increasing $K$ alone cannot overcome a large partial-credit bias.

\subsection{Motivation for the Partition-Style Penalty}
\label{app:partition-loss-motivation}

This section gives a mathematical motivation for the partition-style Yes-rate penalty used in \softsverl{}.
The result assumes the standard Gibbs form of the KL-regularized optimal policy and focuses on the reward gradient induced by a checklist-decomposed verifier reward.

\begin{proposition}[Checklist-decomposed reward gradient]
\label{prop:partition-loss-motivation}
Let $x$ be a prompt, $y$ a response, and $C(x)=\{c_1,\ldots,c_K\}$ a checklist.
Let $\rho_\phi(z\mid x,y,c_k)$ be the verifier trace distribution and let $d(z)\in\{0,1\}$ denote the parsed Yes/No decision.
Define the verifier's item-level Yes-rate as
\begin{equation}
\label{eq:item-yes-rate}
r_\phi(x,y,c_k)
:=
\E_{z\sim\rho_\phi(\cdot\mid x,y,c_k)}[d(z)].
\end{equation}
The checklist reward is
\begin{equation}
\label{eq:motivation-checklist-reward}
R_\phi(x,y)
:=
\frac{1}{K}\sum_{k=1}^K r_\phi(x,y,c_k).
\end{equation}
For a fixed verifier reward $R_\phi$, assume the KL-regularized optimal generator has the Gibbs form
\begin{equation}
\label{eq:motivation-gibbs-policy}
\pi_{\theta^*(\phi)}(y\mid x)
=
\frac{1}{Z_\phi(x)}
\pi_{\mathrm{ref}}(y\mid x)
\exp\!\left(\frac{1}{\beta}R_\phi(x,y)\right),
\end{equation}
where
\begin{equation}
\label{eq:motivation-partition-function}
Z_\phi(x)
=
\sum_{y'}
\pi_{\mathrm{ref}}(y'\mid x)
\exp\!\left(\frac{1}{\beta}R_\phi(x,y')\right).
\end{equation}
Let $\mathcal{D}^+$ be a distribution of target-good prompt-response pairs, and define
\begin{equation}
\label{eq:target-likelihood}
\mathcal{L}_{\mathrm{gen}}(\phi)
:=
\E_{(x,y)\sim\mathcal{D}^+}
\left[
\log \pi_{\theta^*(\phi)}(y\mid x)
\right].
\end{equation}
Then
\begin{equation}
\label{eq:decomposed-reward-gradient}
\nabla_\phi \mathcal{L}_{\mathrm{gen}}(\phi)
=
\frac{1}{\beta K}
\sum_{k=1}^K
\left(
\E_{(x,y)\sim\mathcal{D}^+}
[\nabla_\phi r_\phi(x,y,c_k)]
-
\E_x
\E_{y'\sim \pi_{\theta^*(\phi)}(\cdot\mid x)}
[\nabla_\phi r_\phi(x,y',c_k)]
\right).
\end{equation}
\end{proposition}

\begin{proof}
Using Eq.~\ref{eq:motivation-gibbs-policy},
\[
\log \pi_{\theta^*(\phi)}(y\mid x)
=
\log \pi_{\mathrm{ref}}(y\mid x)
+\frac{1}{\beta}R_\phi(x,y)
-\log Z_\phi(x).
\]
Since $\pi_{\mathrm{ref}}$ does not depend on $\phi$,
\[
\nabla_\phi \log \pi_{\theta^*(\phi)}(y\mid x)
=
\frac{1}{\beta}\nabla_\phi R_\phi(x,y)
-
\nabla_\phi \log Z_\phi(x).
\]
From Eq.~\ref{eq:motivation-partition-function},
\[
\nabla_\phi Z_\phi(x)
=
\sum_{y'}
\pi_{\mathrm{ref}}(y'\mid x)
\exp\!\left(\frac{1}{\beta}R_\phi(x,y')\right)
\frac{1}{\beta}\nabla_\phi R_\phi(x,y').
\]
Dividing by $Z_\phi(x)$ gives
\[
\nabla_\phi \log Z_\phi(x)
=
\frac{1}{\beta}
\E_{y'\sim\pi_{\theta^*(\phi)}(\cdot\mid x)}
[\nabla_\phi R_\phi(x,y')].
\]
Therefore,
\[
\nabla_\phi \mathcal{L}_{\mathrm{gen}}(\phi)
=
\frac{1}{\beta}
\left(
\E_{(x,y)\sim\mathcal{D}^+}
[\nabla_\phi R_\phi(x,y)]
-
\E_x
\E_{y'\sim\pi_{\theta^*(\phi)}(\cdot\mid x)}
[\nabla_\phi R_\phi(x,y')]
\right).
\]
Substituting
\[
R_\phi(x,y)
=
\frac{1}{K}\sum_{k=1}^K r_\phi(x,y,c_k)
\]
and exchanging the finite sum with the expectations yields Eq.~\ref{eq:decomposed-reward-gradient}.
\end{proof}

The second term in Eq.~\ref{eq:decomposed-reward-gradient} comes from differentiating the partition function $Z_\phi(x)$.
It is an item-level policy-sample correction: under the reward-induced generator distribution, it pushes down reward assigned to generated responses.
The partition-style penalty in \softsverl{} is a selected approximation to this term.
The exact correction would require an expectation over prompts, policy samples, and all checklist items:
\begin{equation}
\label{eq:full-partition-correction}
\frac{1}{K}
\sum_{k=1}^K
\E_x
\E_{y'\sim \pi_{\theta^*(\phi)}(\cdot\mid x)}
[\nabla_\phi r_\phi(x,y',c_k)].
\end{equation}
In practice, \softsverl{} replaces $\pi_{\theta^*(\phi)}$ with the current generator $\pi_\theta$ and makes the expectation tractable by selecting a subset $\mathcal{S}$ of on-policy response-item contexts.
This gives the partition objective
\begin{equation}
\label{eq:partition-loss}
\mathcal{J}_{\mathrm{part}}(\phi)
:=
\E_{(x,y,c)\in\mathcal{S}}
\left[
r_\phi(x,y,c)
\right].
\end{equation}
The gradient $\nabla_\phi\mathcal{J}_{\mathrm{part}}(\phi)$ is therefore a selected finite-sample approximation to the policy-sample expectation in Eq.~\ref{eq:full-partition-correction}.
Since \softsverl{} subtracts $\mathcal{J}_{\mathrm{part}}$ in the shared objective, this term has the same sign as the negative policy-sample term in Eq.~\ref{eq:decomposed-reward-gradient}.
Thus, the partition loss can be viewed as a tractable, selected approximation to the log-partition correction.
The subset $\mathcal{S}$ matters because applying this correction to all on-policy items could also suppress genuine successes; \softsverl{} applies it only to selected contexts where Yes-rate inflation is the relevant failure mode.

The first term in Eq.~\ref{eq:decomposed-reward-gradient} increases item-level reward on target-good responses.
This term is written as if we had access to a distribution of target-good full responses.
In \softsverl{}, we do not assume such completions are available for every prompt.
Instead, we generate responses from the initial or current policy and label them at the checklist-item level, producing gold and replay tuples $(x,y,c,\ell)$.
We approximate the positive part of the term with items labeled $\ell=1$, and use items labeled $\ell=0$ to train the verifier for item-level correctness.

\end{document}